\documentclass{article}

\usepackage{iclr2017_conference,times}

\setcitestyle{round}

\usepackage{amsmath}
\usepackage{amsfonts}
\usepackage{amsthm}
\usepackage{caption}
\usepackage{appendix}
\usepackage[T1]{fontenc}
\usepackage{footmisc}
\usepackage{graphicx}
\usepackage[utf8]{inputenc}
\usepackage{nicefrac}
\usepackage{soul}
\usepackage{subcaption}
\usepackage{url}
\usepackage{verbatim}
\usepackage{xcolor}

\definecolor{darker}{rgb}{0,0.15,0.8}
\usepackage[colorlinks, urlcolor=darker, citecolor=darker, linkcolor=darker]{hyperref}

\title{Deep Learning Scaling is Predictable, Empirically}

\author{
  Joel Hestness, Sharan Narang, Newsha Ardalani, Gregory Diamos, Heewoo Jun, \\
  \And Hassan Kianinejad, Md. Mostofa Ali Patwary, Yang Yang, Yanqi Zhou\\
  \\
  \texttt{\{joel,sharan,ardalaninewsha,gregdiamos,junheewoo,hassankianinejad,}\\
  \texttt{patwarymostofa,yangyang62,zhouyanqi\}@baidu.com}\\
  \\
  Baidu Research \\
}

\begin{document}

\maketitle

\begin{abstract}

Deep learning (DL) creates impactful advances following a virtuous recipe: model architecture search, creating large training data sets, and scaling computation. It is widely believed that growing training sets and models should improve accuracy and result in better products. As DL application domains grow, we would like a deeper understanding of the relationships between training set size, computational scale, and model accuracy improvements to advance the state-of-the-art.

This paper presents a large scale empirical characterization of generalization error and model size growth as training sets grow. We introduce a methodology for this measurement and test four machine learning domains: machine translation, language modeling, image processing, and speech recognition. Our empirical results show power-law generalization error scaling across a breadth of factors, resulting in power-law exponents---the "steepness" of the learning curve---yet to be explained by theoretical work. Further, model improvements only shift the error but do \textit{not} appear to affect the power-law exponent. We also show that model size scales sublinearly with data size. These scaling relationships have significant implications on deep learning research, practice, and systems. They can assist model debugging, setting accuracy targets, and decisions about data set growth. They can also guide computing system design and underscore the importance of continued computational scaling.


\end{abstract}

\section{Introduction}
\label{sec:introduction}

The deep learning (DL) community has created impactful advances across diverse application domains by following a straightforward recipe: search for improved model architectures, create large training data sets, and scale computation. This recipe helps improve user experience and product adoption, which drives increased DL development investments in existing and emerging application domains. As data sets grow and new application domains proliferate, we would like to focus our development efforts through a deeper understanding of how the recipe parts coordinate to drive the most valuable product improvements.

Breaking down the recipe, we note challenges in the search for improved model architectures. Model search can create important new insights and publications, which, in turn, improve products that use the novel models. However, model architecture advances often depend on unreliable epiphany; advances usually involve complex or creative reframing of the modeling problem, and research often involves large-scale hyperparameter search with some serendipity.

As a lower-risk complement to model architecture search, it is important that we investigate the other two recipe parts---creating large training sets and scaling computation---where we may have more control over progress. It is widely believed that simply using more data to train larger models should improve accuracy. We would like to better analyze the relationships between training set scale, computational scale, and model accuracy improvements. In particular, accurately predicting generalization error scaling with training set size would provide a powerful tool for estimating the costs---in data and compute requirements---for advancing state-of-the-art (SOTA).

Although prior works analyze sample complexity requirements to reach a desired generalization error, they appear insufficient to accurately predict error scaling for real applications. Many studies theoretically predict that generalization error "learning curves" take a power-law form, $\varepsilon(m) \propto \alpha m^{\beta_g}$. Here, $\varepsilon$ is generalization error, $m$ is the number of samples in the training set, $\alpha$ is a constant property of the problem, and $\beta_g = -0.5$ or $-1$ is the scaling exponent that defines the steepness of the learning curve---how quickly a model family can learn from adding more training samples\footnote{Note: learning curves measure how much training data a model family requires to reach a particular accuracy. They are different from training and validation curves, which measure number of training iterations a model needs to learn a particular data set.}. Unfortunately, in real applications, we find empirically that $\beta_g$ usually settles between $-0.07$ and $-0.35$, exponents that are unexplained by prior theoretical work.

This paper presents the largest scale empirical characterization of learning curves to date that reveals broadly that DL generalization error does show power-law improvement, but with exponents that must be predicted empirically. We introduce a methodology to accurately predict generalization error and model size scaling with increased training set size. We use this methodology to estimate scaling relationships for six deep neural network models across four application domains: machine translation, language modeling, image classification, and speech recognition.

Our results show that power-law learning curves exist across all tested domains. Although different applications yield different power-law exponents and intercepts, these learning curves exist across a broad range of models, optimizers, regularizers, and loss functions. Improved model architectures and optimizers can improve the power-law intercept, but not the exponent; models for a single domain show the same learning curve steepness. Finally, we find that models transition from a small training set region dominated by best guessing to a region dominated by power-law scaling. With sufficiently large training sets, models will saturate in a region dominated by irreducible error (e.g., Bayes error).

Further, we describe the significant implications of predictable accuracy and model size scaling. For DL practitioners and researchers, learning curves can assist model debugging and predict the accuracy targets for improved model architectures. Our results suggest an opportunity for redoubled effort to theoretically predict or interpret learning curve exponents. Operationally, predictable learning curves can guide decision-making about whether or how to grow data sets. Finally, learning and model size curves can be used to guide system design and expansion, and they underscore the importance of continued computational scaling.

This paper is organized as follows: We begin with a review of related work measuring relationships between data set size, model size, and model capacity in Section~\ref{sec:related}. Section~\ref{sec:methodology} describes our methodology to establish such relationships for four machine learning domains, and Section~\ref{sec:results} shows those relationships for the four domains. Section~\ref{sec:discussion} discusses some implications of our scaling results.

\section{Related Work}
\label{sec:related}

Since our objective is to accurately predict generalization error and model size scaling with increased training set size, we start with a review of prior theoretical and empirical work to see if they are adequate to predict the behavior we see. Prior work investigates generalization error improvements as sample complexity increases using three approaches: theoretically bounding generalization error scaling, theoretically estimating the expected generalization error, and empirically collecting generalization error for single applications. Prior work also deeply analyzes the theoretical model capacity, suggesting the model size required to fit training data.

Unfortunately, although these prior works offer general guidance, they are not able to explain our empirical results. To the best of our knowledge, this paper is the first to empirically characterize learning curve and model size scaling trends for a broad range of application domains and models.

%
\subsection{Generalization Error Scaling with Data: Learning Curves}

We start with a survey of studies that investigate learning curves. Most of these works show power-law generalization error scaling ($\varepsilon(m) \sim \alpha m^{\beta_g}$) with exponent $\beta_g = -0.5$ or $-1$.

\textbf{Bounding Generalization Error:}
Many prior works provide theoretical bounds on the sample complexity to ensure particular generalization error. Early theoretical work defines a framework for bounding generalization, but makes weak assumptions that cause the predicted generalization error to be very loose (\cite{haussler:valiantsframework:ai:1988}). Early follow-on research tightens the bounds by relating sample complexity to generalization error through the Vapnik-Chervonenkis (VC) dimension of the target concept class (\cite{ehrenfeucht:sampcomplxlowerbd:infocomp:1989,blumer:learnvcdim:jacm:1989,haussler:rigorousbounds:machinelearning:1996}). All of these bounds show power-law relationships under certain assumptions, such as the hypothesis space must contain at least one model that can correctly fit the data or the training data size must be much larger than the capacity of the models. These assumptions are often too strict for real applications, so the bounds are usually loose or even vacuous. Recent work by \cite{dziugaite:nonvacuousbounds:icml:2017} tightens bounds for the common real application setting that model size is larger than the number of samples in the data set. However, despite the breadth of prior sample complexity bounds, we have yet to find straightforward bounds that explain our empirical results.

\textbf{Estimating Expected Generalization Error:} Prior work also evaluates the expected generalization error in certain contexts. Using statistical mechanics approaches, \cite{amari:fourlearningcurves:neuralcomp:1992} and \cite{amari:universallearningcurve:jnn:1993} show that as sample complexity grows, generalization error should decline as a power-law $\varepsilon(m) \sim \alpha m^{\beta_g}$ with $\beta_g = -0.5$, $-1$, or $-2$. These trends depend on assumptions about the problem and give the expectation across all possible data distributions. Amari and others show that similar expectations hold for certain models, such as single- and multi-layer perceptrons, and committees of networks (\cite{gyorgyi:learningarule:nnandspin:1990,seung:mechlearningexamples:physreva:1992,schwarze:mechlargecommittee:nips:1993,amari:entropicloss:neuralcomp:1993}). Appendix \ref{appendix:powerlawproof} adds to this corpus, showing that a counting model used to predict the probability of a weighted coin-flip converges with the power-law exponent of $\beta_g = -0.5$.

Despite the breadth of prior work estimating the expected generalization error in various contexts, the empirical results in this paper show yet unexplained power-law exponents between $\beta_g = -0.07$ and $-0.35$ on various real world problems. Our results suggest an opportunity for redoubled effort to theoretically justify our empirical generalization scaling trends.


\textbf{Empirical Generalization Error Scaling:}
A few prior studies empirically investigate the way generalization error scales with training data size. Using a methodology similar to the one we propose below, \cite{banko:verylargenld:acl:2001} test a language modeling problem (confusion set disambiguation) trained using subsets of a billion-word corpus of text. Their results appear to show power-law scaling of the average disambiguation validation error. In speech recognition, \cite{amodei:ds2:icml:2016} show word error rate improvement for a Deep Speech 2 model on varying sizes of training data. They use a fixed model size of 68M parameters and show power-law WER gains from increased data. \cite{sun:dataeffective:iccv:2017} show image classification accuracy improves with training data size, but curiously, they conclude that accuracy "increases logarithmically based on volume of training data size".

Although some prior works study generalization error scaling trends empirically, the community has yet to definitively conclude that power-law error scaling should exist across most DL domains.


%
\subsection{Model Capacity Required to Fit Data}

Prior studies propose various measures of model capacity based on a model's organization and parameterization, and these measures hint at the model size required to fit a training set. We expect that number of model parameters to fit a data set should follow $s(m) \propto \alpha m^{\beta_p}$, where $s(m)$ is the required model size to fit a training set of size $m$, and $\beta_p \in [0.5, 1]$.

Vapnik and Chervonenkis defined the VC dimension of a model as the cardinality of the largest set of data points that a model can shatter (\cite{vapnik:vcdim:ieeenn:1998}). Follow-on work uses data complexity measures to estimate the structure of model families that might fit the data (\cite{bartlett:rademachergaussian:jmlr:2002}). Recent work also defines bounds on the VC dimension of particular deep neural network models, including showing that recurrent neural network models have the same effective capacity if the optimization scheme is well-tuned and training runs long enough (\cite{harvey:nearlytightvcdim:jmlr:2017,dziugaite:nonvacuousbounds:icml:2017,collins:modelcaptrain:iclr:2017}). 

Prior work to empirically estimate model scaling with training set size is very sparse. The \cite{banko:verylargenld:acl:2001} confusion set disambiguation work claims that the model size required to fit the data grows "log-linearly". We estimate that their Winnow and memory-based models grow with the same power-law exponent to larger data sets, $\beta_p \approx 0.72$.

While these theoretical and empirical results offer significant insight about required model sizing, recent work has noted the need for more practical guidance (\cite{zhang:rethinkgeneral:arxiv:2017,arpit:memorizationdeep:icml:2017,smith:bayesiangeneralize:arxiv:2017,kawaguchi:dlgeneralization:arxiv:2017}). These studies show that while model capacity might explain a model's ability to memorize training examples, capacity may not adequately explain the model's ability to generalize to new examples. Rather than reason through these complexities, it is currently easier for researchers and practitioners to over-parameterize models to fit training data.

\section{Measuring Model Accuracy and Size Scaling with Training Data Size}
\label{sec:methodology}

With the general guidance from prior work in mind, we focus our attention on accurately estimating learning curves and model size scaling trends. We measure the effects of scaling data size on generalization error and model size using the following methodology. The general process is to select state-of-the-art (SOTA) models and to train "hyperparameter-reduced" versions of these models on successively larger subsets (shards) of a training set to see how the accuracy of the model grows with training set size.

First, for each of the machine learning domains, we survey recent work to find the model architectures that show SOTA generalization error on a large data set. Here, a "large data set" is a training set that could be reduced in size by 2-3 orders of magnitude and still be significant enough to perform valuable model architecture studies. We select more than one model architecture for some ML domains to compare their scaling behaviors.

\textbf{Data sets:}
Given a SOTA model architecture, $M$, and a training set, $T$, we set up experimental infrastructure as follows. First, we ensure that $T$ is randomly shuffled to maximize the likelihood that shards of $T$ will have similar data distribution to $T$. We then subdivide $T$ into shard sizes that span 2-3 orders of magnitude in steps of roughly 2$\times$ (e.g., $T_0$ is 0.1\% of $T$, $T_1$ is 0.2\%, $T_2$ is 0.4\%, etc.). We define a single validation set, $V$, which is used to score all models (even trained on different shard sizes of $T$), such that $\forall i, V \cap T_i = \emptyset$. $V$ must be sufficiently large to approximate true generalization error with low variance. We use either the validation set available with training data, or if such a validation set is not available, we use a hold-out subset of $T$ that does not overlap with any of the $T$ shards.

We ensure that the metric used to measure the size of the data set accurately represents the observable size of the training set. For example, character language models truncate time-series input sequences at a maximum length and discard the rest of the sequence from the training set. In such situations, the data set size must only count the portion of the input sequences observed by the model during a full training run.

\textbf{Model setup:}
We first replicate the SOTA results on $T$, setting hyperparameters of $M$ as described in the corresponding prior works. Next, we aim to understand the importance of model capacity to fit a training set, so we control our experiments by removing regularization schemes that might reduce the model's effective capacity (e.g., weight decay). With this simplification, we can inspect validation curves to find the smallest model size that is able to overfit each shard size of $T$. For models that achieve SOTA results while underfitting on their training set, we reduce the the training set size to a scale that the model can overfit.

We aim to find a model variant of $M$ that best fits $V$ when trained on the smallest shard, $T_0$, and to find this variant, we reduce hyperparameters of $M$ and perform a grid search. We generate a set of model candidates, $\mathcal{M}_0 = \{M_{0;0}, M_{0;1}, M_{0;2},\ldots\}$, by constraining $M$'s model capacity, changing hyperparameters such as layer count, hidden node count, etc. From this search, we find $\hat{M}_0 = arg~min_{M_{0;j} \in \mathcal{M}_0}\left(\mathcal{L}(M_{0;j},V)\right)$, which gives the best validation loss, $\mathcal{L}$, on $V$ when trained on $T_0$.

\textbf{Training procedure:}
Finally, with best-fit models defined for the smallest and largest shards of $T$, we perform a stochastic Monte Carlo grid search to find the best-fit hyperparameter-reduced models as we step through successively larger shards of $T$. Specifically, given the best-fit model size for shard $T_i$, we project forward to shard $T_{i+1}$---often increasing model sizes linearly or sublinearly in the shard size---to define the candidate set of models $\{M_{i+1;0}, M_{i+1;1}, M_{i+1;2},\ldots\}$. We train these models in search of the models that best fit the validation set. In most cases, we also search over optimization parameters such as batch sizes and learning rates, and sometimes re-run training with different random seeds to aid the search.

We report validation losses that are sums or unweighted averages over distance metrics measuring the error per model-predicted output. This loss structure appears important to the predictivity of the resulting learning curves. Depending on problem domain, error metrics include cross-entropy, $L^p$ norms, and classification error. In many cases, training optimizes a different loss function than we report as the validation loss (see Appendix \ref{appendix:domaindetail} for more details).


\section{Data Set and Model Size Scaling Relationships}
\label{sec:results}

In this section, we present empirical results showing how increasing training data size results in power-law scaling of generalization error and required model size to fit the training set for four domains: machine translation, language modeling, image classification, and speech recognition\footnote{Data presented here was collected from roughly 50 years worth of GPU time.}. These power-law relationships hold for each ML domain and across various model architectures, optimizers, and loss metrics. In many cases, we also find that model size growth with data set size grows sublinearly. Throughout this section, we specifically refer to power-law exponents for generalization error ($-0.5 \leq \beta_g < 0$ in $\varepsilon(T_i) = \alpha |T_i|^{\beta_g}$) and number of model parameters ($0.5 \leq \beta_p < 1.0$ in model size $ = \alpha|T_i|^{\beta_p}$).

%
\subsection{Neural Machine Translation}

We start our learning curve investigation with a case study in neural machine translation (NMT). Translation converts text input in one natural language to output text in another language. Relative to other DL domains, NMT has low-dimensional input and output spaces, and can be trained with large labeled data sets. Our results show learning curve character similar to theoretical predictions, though the power-law exponents are smaller (i.e., $\beta_g \approx -0.128$ rather than $-0.5$).

To test NMT, we train a SOTA sequence-to-sequence model with global attention (\cite{luong:globalattention:emnlp:2015}) on the 2016 Conference on Machine Translation (WMT’16) German-to-English data set. We use a publicly available implementation of this architecture in OpenNMT (\cite{klein:opennmt:acl:2017}). The encoder contains two layers of bidirectional LSTMs, and the decoder contains the attention layer and stack of LSTM layers. To simplify training this SOTA model, we remove ensembling and data augmentation techniques (\cite{sennrich:edinburghnmtwmt16:arxiv:2016}).

To scale model sizes, we tie LSTM input and hidden state sizes together, and change them so that the total parameter count decreases roughly linearly with data set size. We use Adam to optimize per-sequence cross-entropy loss and report the per-token classification error. We select models using the newstest2015 validation set, and we use the other newstest development sets from 2009 to 2013 for evaluation. Results presented here are with dropout rate of 0.2, though we tested without dropout and found similar learning curve exponents.

We clean and tokenize the data set using Moses (\cite{koehn:moses:jacl:2007}) as described by \cite{luong:seq2seqtutorial:tf:2017}. We use the byte-pair encoding (BPE) method described by \cite{sennrich:bpe:arxiv:2016} to build a shared word-piece vocabulary between English and German. After preprocessing, the training set includes 4.5 million training sequences with roughly 130 million tokens in each language. We uniformly randomly shuffle the training data and sample training shards as described in Section~\ref{sec:methodology}.

\begin{figure}
  \centering
  \begin{subfigure}[b]{0.5\textwidth}
    \includegraphics[width=\textwidth]{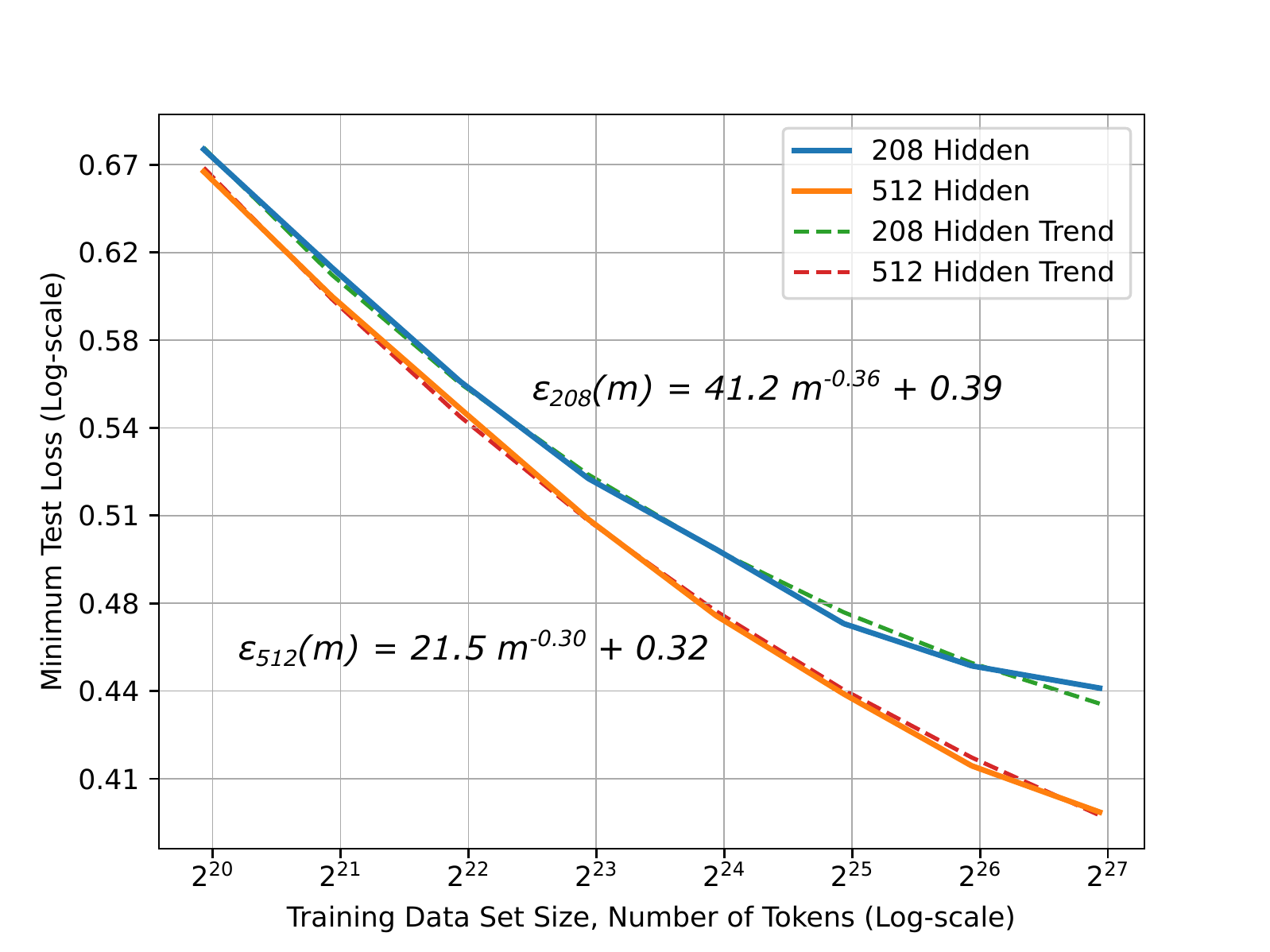}
    \label{fig:nmt_model_learn_curves}
  \end{subfigure}
  \hspace{-10pt}
  \begin{subfigure}[b]{0.5\textwidth}
    \includegraphics[width=\textwidth]{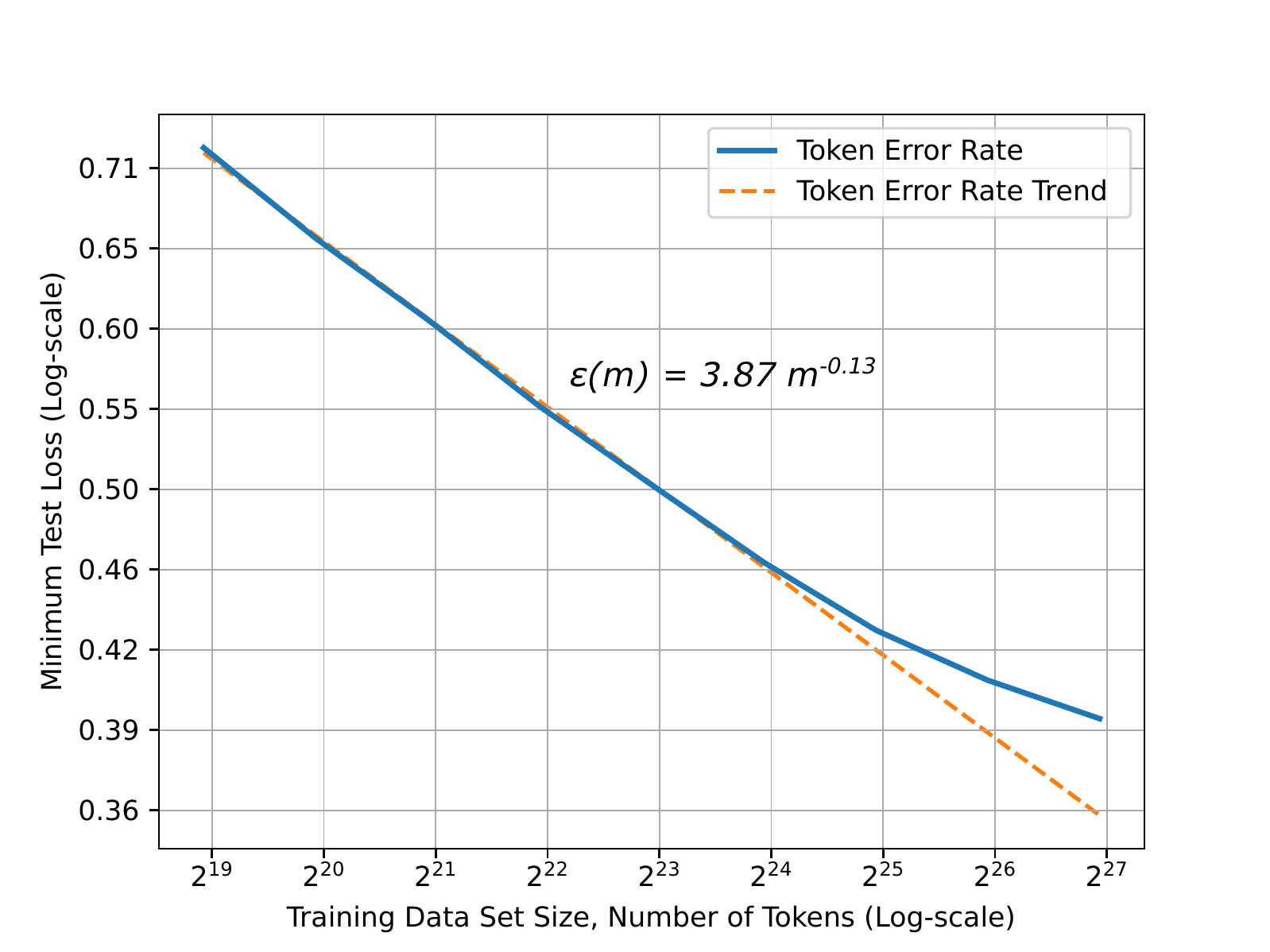}
    \label{fig:nmt_data_gen}
  \end{subfigure}
  \caption{Neural machine translation learning curves. Left: the learning curves for separate models follow $\varepsilon(m) = \alpha m^{\beta_g} + \gamma$. Right: composite learning curve of best-fit model at each data set size.}
  \label{fig:nmt_results}
\end{figure}

In our initial tests, we aim to replicate theoretical results as closely as possible. Prior theoretical work indicates that the expected classification error learning curve for a single model family (i.e., of fixed capacity) is a power-law with exponent $\beta_g = -0.5$ (\cite{amari:fourlearningcurves:neuralcomp:1992}). Further, prior work predicts that as a model runs out of capacity on larger data sets, the error should plateau, resulting in a power-law + constant, $\varepsilon(m) \sim \alpha m^{\beta_g} + \gamma$, where $\gamma$ is the error when the model family has exhausted its capacity.

Indeed, we find that learning curves for a single model family can be closely represented by a power-law + constant. However, we find that $\beta_g$ is smaller in magnitude than $-0.5$. We start by training fixed size models on each of the training shards. The left plot in Figure~\ref{fig:nmt_results} shows the learning curves for two different model sizes with 208 or 512 hidden nodes per LSTM layer (17M and 48M parameters, respectively). Learning curves with $\beta_g = -0.360$ and $-0.300$, respectively, fit the empirical results with less than $0.6\%$ relative root mean square error.

For these experiments, we use controls as close to theoretical assumptions as possible. We use the same loss function, classification error. To approximate the generalization error expectation calculations without an excessive number of training runs, we select models using the median minimum validation error across multiple training runs with separate random seeds. We cannot control for factors such as the assumed data distribution or ensure that the model family contains a model that can correctly represent the data generating function. These factors might account for a portion of the gap from theoretical to empirical $\beta_g$.

Unlike these initial tests, DL practitioners and researchers often grow model sizes as training data grows to ensure sufficient capacity. They would rather see a composite learning curve representing the best-fit model at each training set size. The right plot in Figure~\ref{fig:nmt_results} shows the composite learning curve for NMT. The best-fit results form a longer power-law region. We find that $\beta_g$ is even smaller than the single-model learning curves; if we project forward, $\beta_g$ would be approximately $-0.128$. The rest of our results aim to characterize the steepness of these composite best-fit learning curves.

We also note that as training set sizes grow, optimization becomes more difficult and models run out of capacity, so the empirical error tends away from the power-law trend. This divergence is common across domains, as we show below, and we would need to perform a more exhaustive hyperparameter search to find results closer to the existing power-law.

%
\subsection{Language Modeling}

Language models (LMs) aim to predict probability distributions for the next character, word, or other textual grams conditioned on a previous sequence of input text. LMs are very important model features for domains such as speech recognition and machine translation, helping to identify most probable sequences of grams. Similar to NMT, LMs have low-dimensional input and output spaces, and can be trained with very large labeled sets.

LM learning curves and model size scaling relationships are the most robust; word and character language models show clear and predictable power-law learning curves, and the power-law exponents tend to be small ($\beta_g \in [-0.09, -0.06]$). These small exponents indicate that current language models will require significantly more data to significantly improve accuracy. The word and character models that give the best generalization error grow sublinearly in the training set size ($\beta_p \approx 0.7$).

\subsubsection{Word Language Models}

We train LSTM-based word LMs that were early SOTA models as described in~\cite{jozefowicz:lmlimits:arxiv:2016} with some small changes. To reduce the computational requirements of the models, we restrict the vocabulary to the top 10,000 most frequent words in the Billion Word Dataset (\cite{chelba:1bw:arxiv:2013}). The networks are 2- or 4-layer LSTMs with the same number of hidden weights in each layer, and we scale the number of layer weights to modulate the model size and find the best fit model for each training shard size. We also compare LSTMs against Recurrent Highway Networks (RHNs) described in \cite{zilly:rhns:icml:2017}. Specifically, we train single-layer, depth 5 RHNs to see if the different network organizations show different generalization trends. We use a stochastic gradient descent optimizer (SGD) with per-sequence cross-entropy loss, and we report per-predicted-word average cross-entropy loss. We do not use dropout. We train the models on shards ranging from 0.1\% up to 40\% of the Billion Word Dataset.

\begin{figure}
  \centering
  \begin{subfigure}[b]{0.5\textwidth}
    \includegraphics[width=\textwidth]{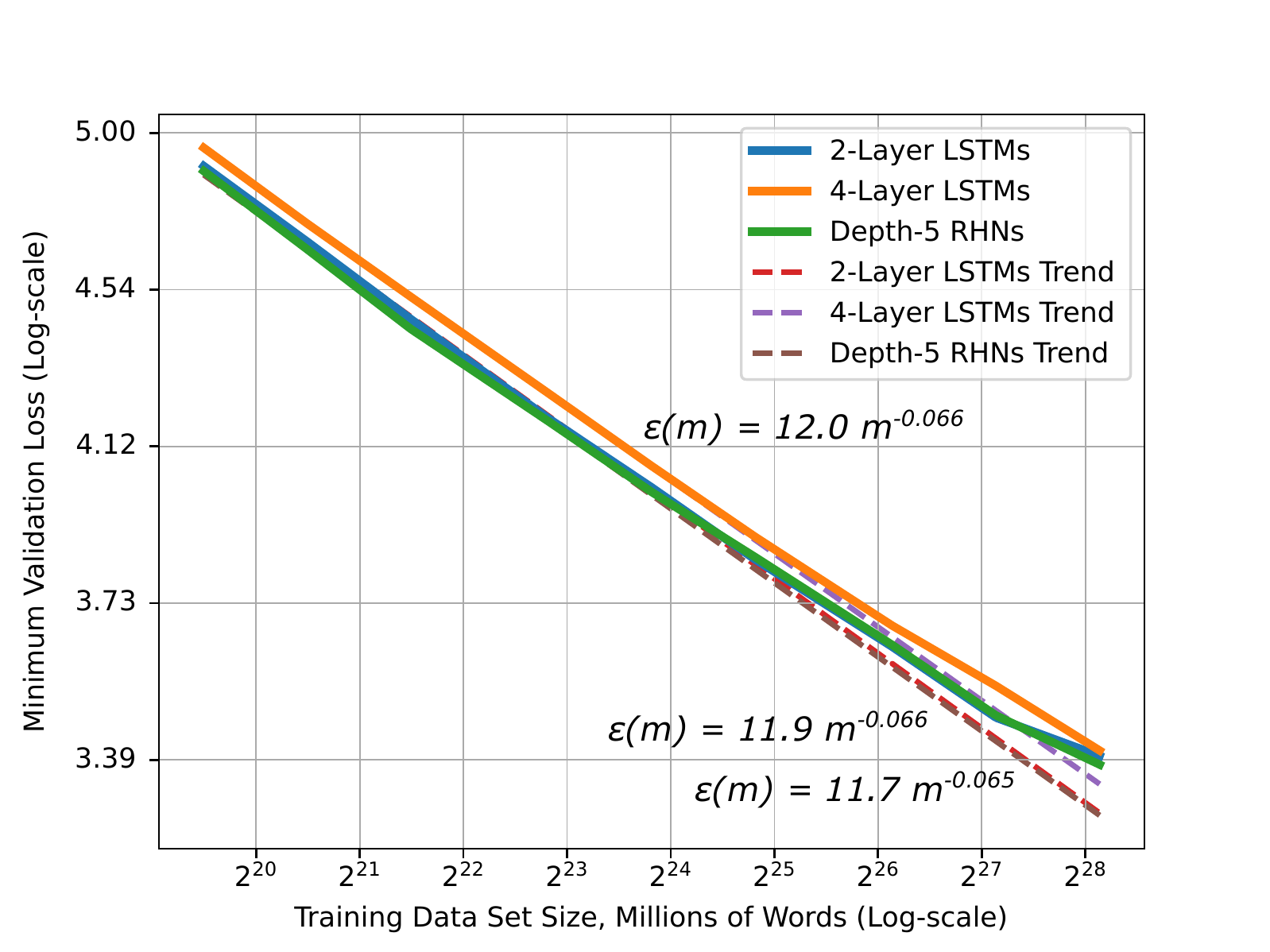}
    \label{fig:word_data_gen}
  \end{subfigure}
  \hspace{-10pt}
  \begin{subfigure}[b]{0.5\textwidth}
    \includegraphics[width=\textwidth]{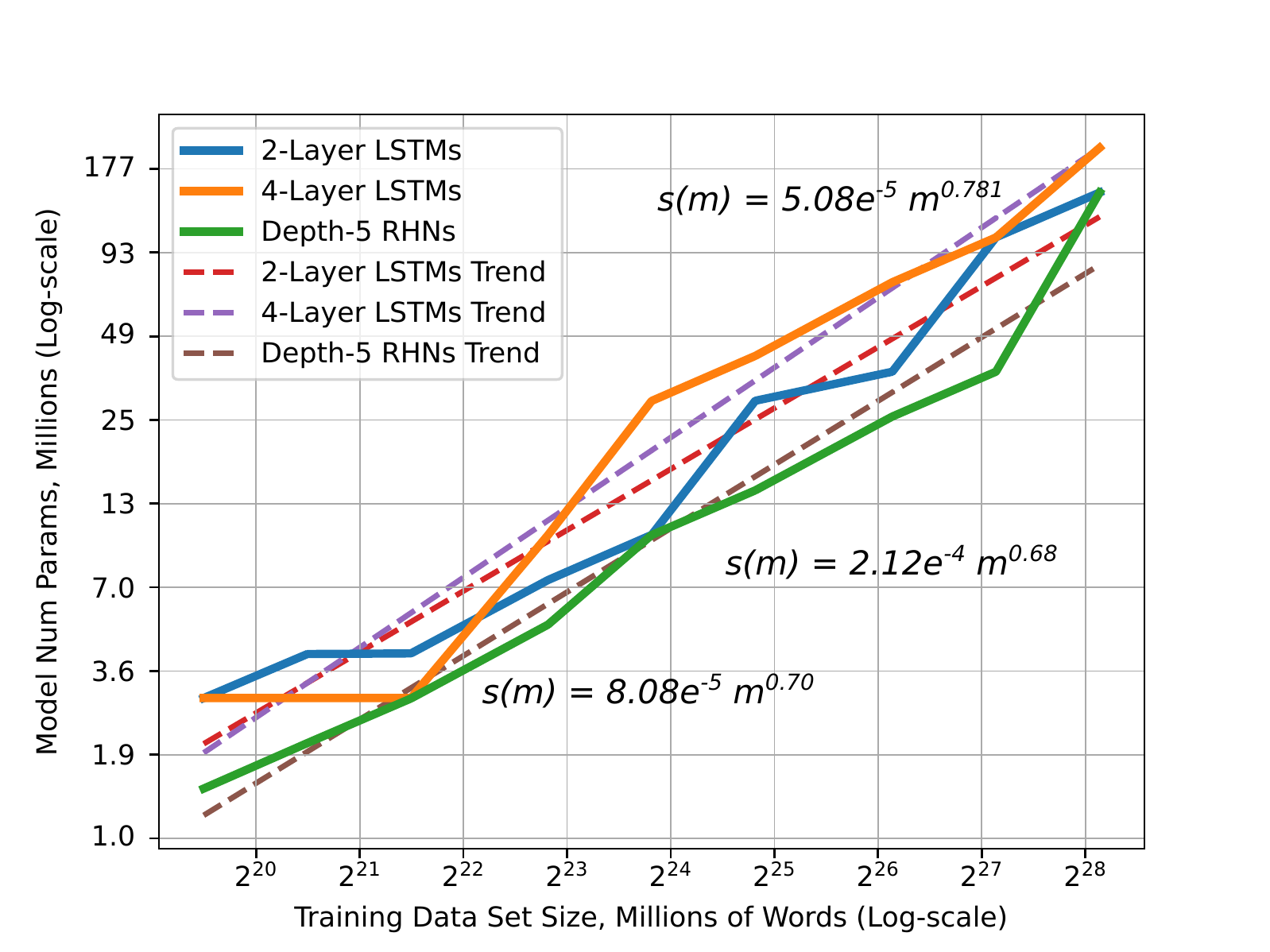}
    \label{fig:word_model_size}
  \end{subfigure}
  \caption{Learning curve and model size results and trends for word language models.}
  \label{fig:word_results}
\end{figure}

Figure~\ref{fig:word_results} shows the learning curve and model size results for LSTM and RHN word language models. First, the loss scaling relationships are smooth power-law functions of the data set size with almost exactly the same exponents: $\beta_g = -0.0656 \pm 1\%$. Again, larger models have more difficulty optimizing to fit the larger training sets. For word LMs, we invest more in batch size and learning rate tuning to find the best model on these larger training shards. The tuned models settle at or just above the power-law trend, suggesting that further hyperparameter search is likely to yield a model on the trend.

Strikingly, although these model architectures differ appreciably, they all show the same learning curve profile characterized by the power-law exponent. Increasing the LSTMs depth from 2 to 4 layers decreases the networks' accuracy by about $1.5\%$, but both model architectures see the same relative loss improvement as we increase training set size. RHNs have significantly different recurrence structure than LSTMs, but show nearly identical learning curves.

Model size results show that best-fit models grow sublinearly in the training shard size. Specifically, the best-fit 2-layer LSTM and depth-5 RHNs model sizes grow roughly with $\beta_p = 0.69 \pm 5\%$. The 4-layer LSTMs show slightly worse scaling with $\beta_p = 0.78$, suggesting they make less effective use of extra parameters on larger data sets. Despite the model size scaling differences, for a given model architecture, we can accurately predict the model size that will best fit increasingly larger data sets.

\subsubsection{Character Language Models}

To test character-level language modeling, we train RHNs of depth 10, which we found to achieve SOTA accuracy on the Billion Word data set. We scale the number of layer weights to modulate the model size and find the best fit model for each training shard size. We use SGD, optimizing for per-predicted-character cross-entropy loss, which we report on the validation set. We also compare SGD against the Adam optimizer to test their effects. The input and output vocabulary includes all alphanumeric characters and common symbols for total size 98. We train the models on shards of 0.01\% up to 4\% of the Billion Word data set.

\begin{figure}
  \centering
  \begin{subfigure}[b]{0.5\textwidth}
    \includegraphics[width=\textwidth]{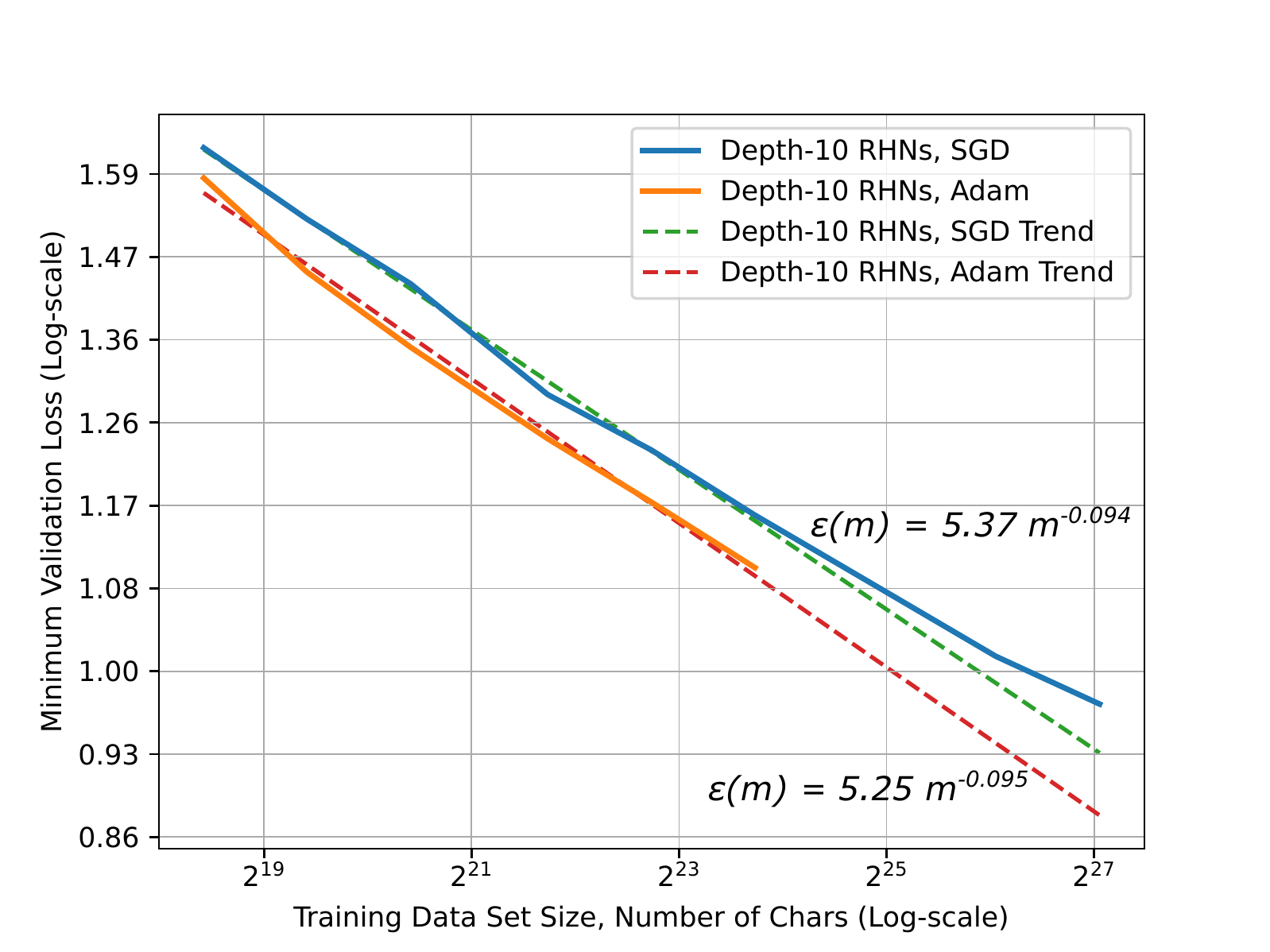}
    \label{fig:char_data_gen}
  \end{subfigure}
  \hspace{-10pt}
  \begin{subfigure}[b]{0.5\textwidth}
    \includegraphics[width=\textwidth]{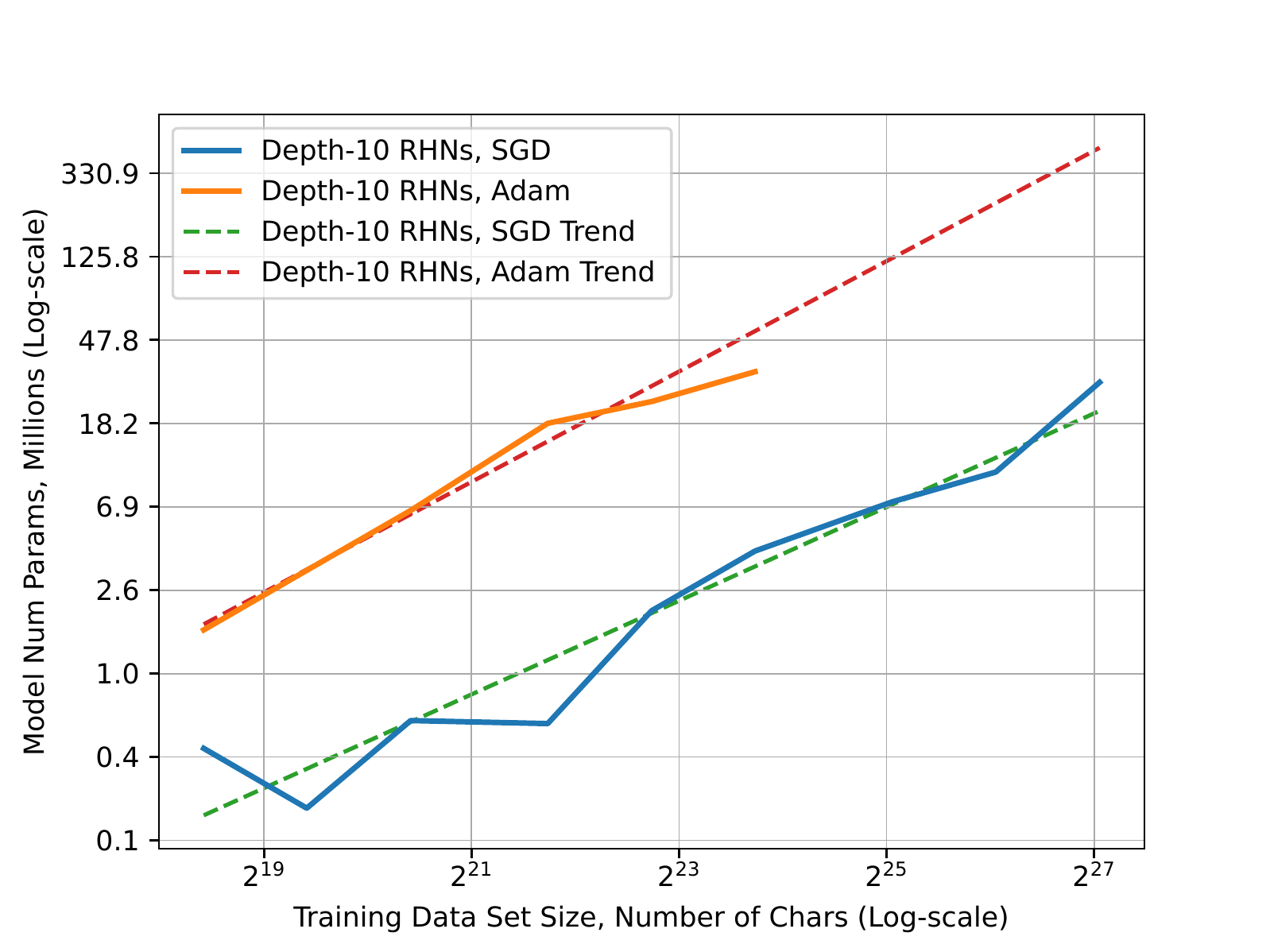}
    \label{fig:char_model_size}
  \end{subfigure}
  \caption{Learning curve and model size results and trends for character language models.}
  \label{fig:char_results}
\end{figure}

Results for character LMs appear substantially similar to word LMs. Figure \ref{fig:char_results} plots the generalization and model size scaling results for character LMs. As with word LMs, generalization improves on a power-law as training data size increases, though the exponent is $\beta_g = -0.0936$ for the SGD optimizer and $\beta_g = -0.0954$ for the Adam optimizer. These power-law exponents are very similar despite the significant optimizer differences---Adam appears to just shift the learning curve down by $\sim 5\%$ relative.

Like word LMs, character LMs also learn significantly more slowly than predicted by theoretical results. Though word and character LMs have some major differences, their learning curve exponent differences indicate that character LMs are able to learn relationships between characters with successively fewer samples than word LMs are able to learn relationships between words.

Character LMs also show sublinear model size growth as data set size increases. Specifically, $\beta_p = 0.78$ for SGD optimized models and $\beta_p = 0.92$ for Adam optimized. Character LMs with the SGD optimizer see similar improvements from increased model size as word LMs, while the Adam optimized models see poorer scaling and require significantly more parameters ($\sim 8$--$11\times$). Still, their learning and model size curves appear predictable.

%
\subsection{Image Classification}

As a comparison to our machine translation and language modeling results---where inputs and outputs are low-dimensional time-series data---we next test image classification, a machine learning domain that aims to identify objects in high-dimensional image data. Image classification is used in applications such as object recognition, image captioning, and tagging video content. Image classification also shows power-law learning curves and model size scaling relationships. We also show that accuracy plateaus near random guessing on very small training sets.

We test ResNets (\cite{he:resnets:cvpr:2016}), which were recently the SOTA architectures for ImageNet classification (\cite{russakovsky:imagenet:arxiv:2015}). ResNets are deep networks built from blocks containing convolutions, nonlinearities, and pooling layers. They have residual connections from the inputs to outputs of most blocks that permit the network to bypass layers. We train and validate ResNets on various shard sizes of ImageNet, ranging from 1 image per class (0.08\% of images) up to 800 images per class (62\%). ImageNet has 1,000 different object classes as outputs.

We start with 5 known variants of ResNets with depths 18, 34, 50, 101, and 152 layers. We first scale the model sizes by changing the number of layers ranging from 10 to 200. To provide even finer-grained model size control, we also change the number of convolution filters using a scaling factor. We scale filter counts proportionally across all convolution blocks with scaling factors 0.0625 to 1.5. We test models with parameter counts ranging from 89K to 121M. We use a Nesterov Momentum optimizer targeting classification cross-entropy loss. We remove weight regularization.

Figure \ref{fig:image_results} shows that various loss calculations follow the power-law learning curves. We report average validation cross-entropy, top-1, and top-5 classification errors. For small training sets---less than roughly $25$ images per class---these error metrics are roughly equal to the model random guessing (i.e., greater than $-log(1/1,000) \approx 6.9$ for cross-entropy, and near $1-(1/1,000) = 99.9\%$ classification error for top-1 and top-5). Models are unable to extract enough information from these small training sets to make many accurate classifications on the validation set. We later describe this as the "small data region".

As long as the training set is large enough, we observe that generalization improves on a power-law, but the power-law exponent is different for each of the reported metrics. The top-1 classification error exponent is $\beta_g = -0.309$. On the other hand, the exponent for top-5 classification error is $\beta_g = -0.488$. Since top-5 classification is a superset of top-1 classification, the top-5 error should improve at least as quickly as top-1, but in fact, the top-5 error improves significantly more quickly as training data size increases. The validation cross-entropy exponent is $\beta_g = -0.35$, but the metric has different range than classification error, so we cannot directly compare their exponents.

Finally, Figure~\ref{fig:image_results} also shows that model size growth is again predictable. The best-fit ResNet models grow following a sublinear curve with exponent $\beta_p = 0.573$. This exponent indicates that they grow more slowly than models in other domains we have tested. However, even on the smallest data sets, ResNets require fairly large models to fit data well, at least 3.4M parameters.

\begin{figure}
  \centering
  \begin{subfigure}[b]{0.5\textwidth}
    \includegraphics[width=\textwidth]{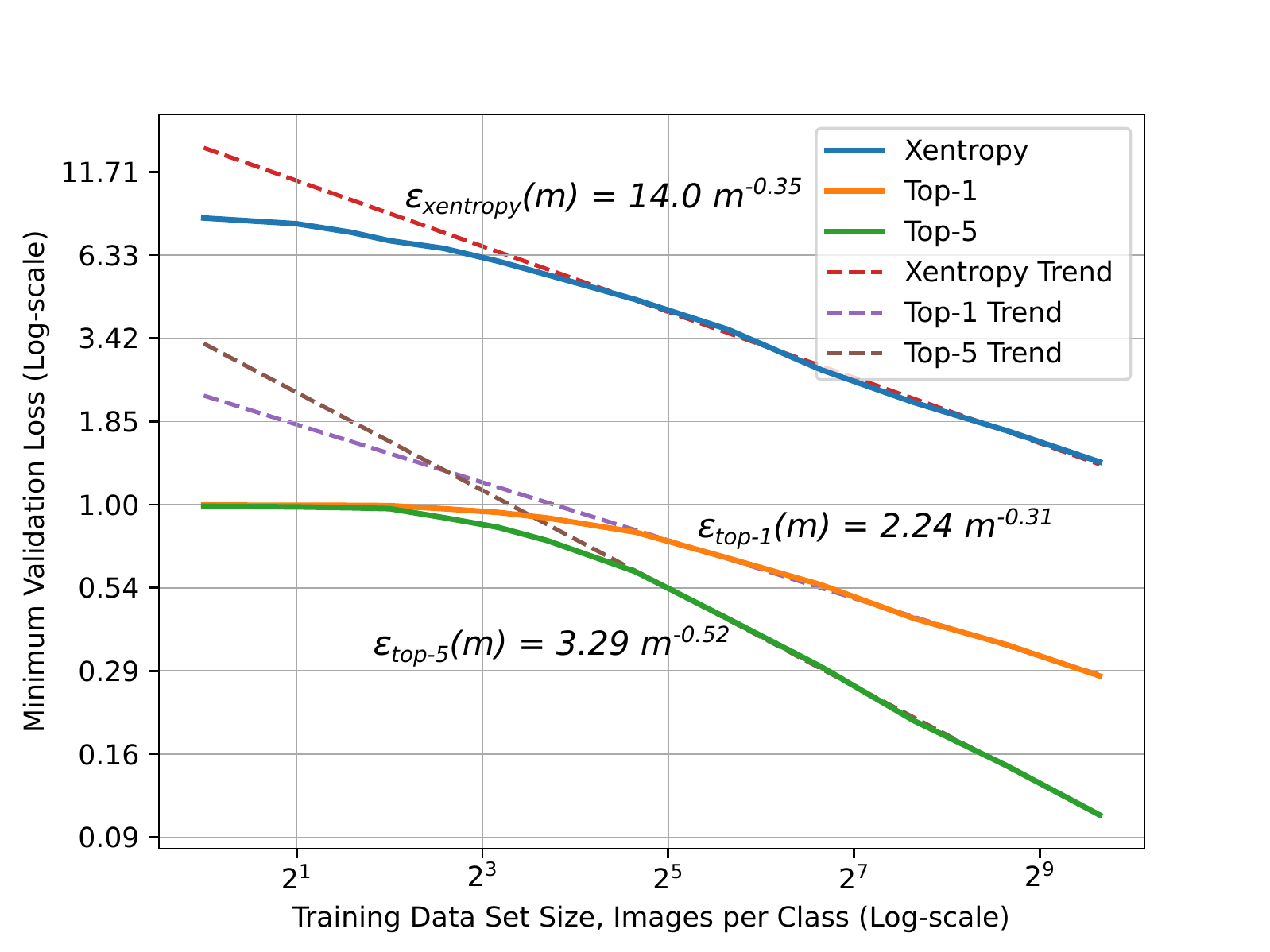}
    \label{fig:image_data_gen}
  \end{subfigure}
  \hspace{-10pt}
  \begin{subfigure}[b]{0.5\textwidth}
    \includegraphics[width=\textwidth]{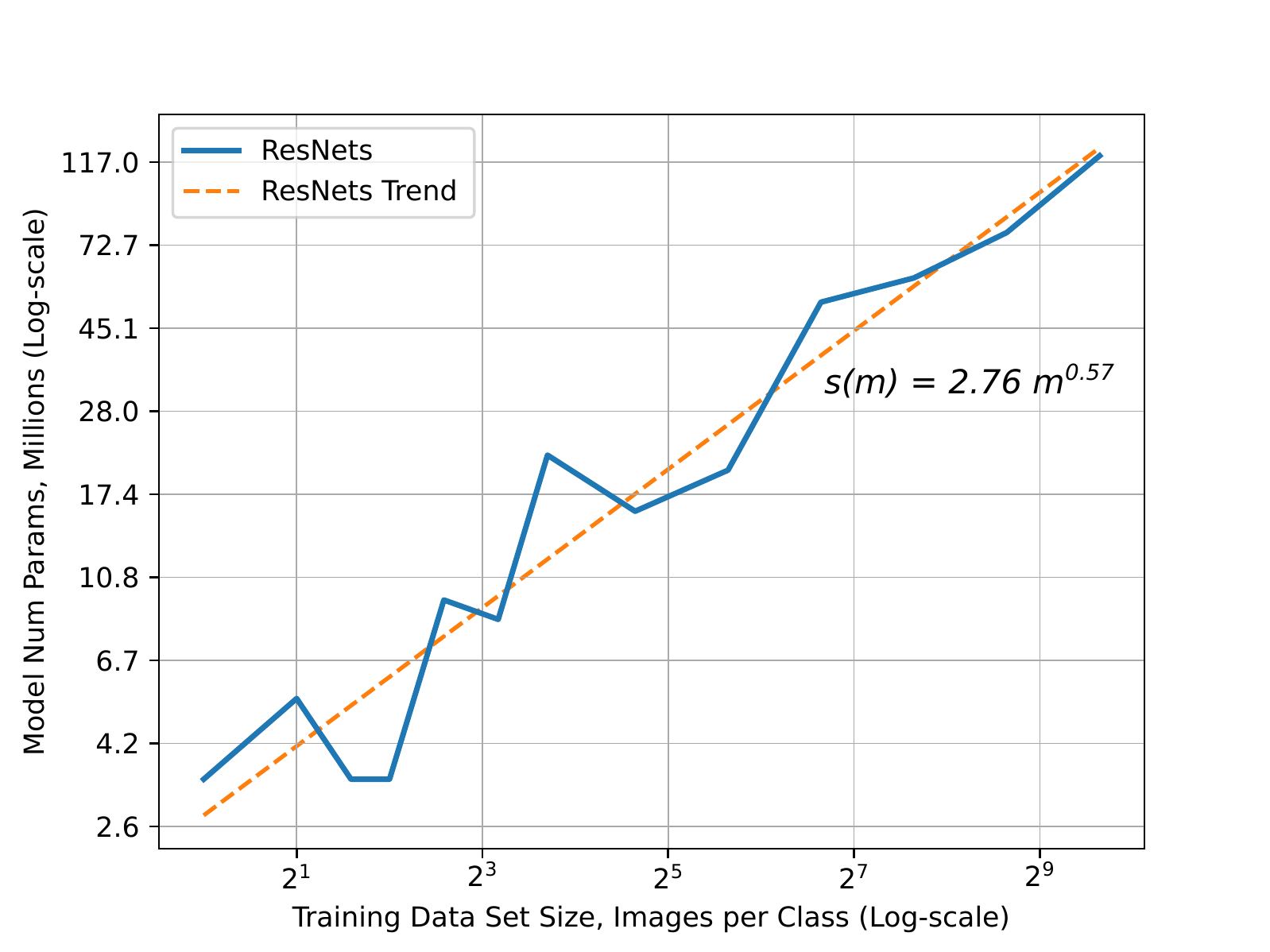}
    \label{fig:image_model_size}
  \end{subfigure}
  \caption{Learning curve and model size results and trends for ResNet image classification.}
  \label{fig:image_results}
\end{figure}

%
\subsection{Speech Recognition}

Speech recognition techniques convert acoustic speech signals into text or commands. Speech recognition is used in diverse applications such as voice-powered machine controls and conversational user interfaces. Recent research has shifted from hand-engineered speech recognition pipelines over to end-to-end deep learning based methods that show promising results (\cite{hannun:deepspeech:arxiv:2014,chorowski:speechattention:nips:2015,amodei:ds2:icml:2016}). Speech recognition provides an interesting contrast to prior domains; speech input data is medium-dimensionality time-series data.

To test trends in speech recognition, we train two recent SOTA models: Deep Speech 2 (DS2) and an attention-based model. The DS2 model (\cite{amodei:ds2:icml:2016}) consists of two 2D convolution layers followed by four bidirectional LSTM recurrent layers. We use Adam to optimize connectionist temporal classification loss (CTC, \cite{graves:ctc:icml:2006}). We compare DS2 against a hybrid attention model similar to those described by \cite{battenberg:speechtxducers:arxiv:2017}. The model has an encoder comprised of three bidirectional LSTM layers with two intermediate max-pooling layers, and a hybrid attention decoder. We use Adam to optimize output sequence average cross-entropy loss. For both models, we remove regularization (weight decay and noise) to observe underfitting or overfitting models.

The inputs to these models are a sequence of log-spectrograms of power normalized audio clips, calculated on 20 ms windows. Outputs are the English alphabet along with the blank symbol. We do \textit{not} include language models for output sequence beam search, and we report per-predicted-output character error rate on the validation set. We train on shards of labeled data set comprising 11,940 hours of speech containing 8 million utterances \cite{amodei:ds2:icml:2016}.

To vary the number of parameters in both the DS2 and attention models, we vary the number of weights in all LSTM layers, so that separate layers have the same number of weights. In the attention model, we also proportionally scale number of weights in the attention LSTM and decoder cells. For the DS2 model, model sizes range between 300K to 193M parameters, and for the attention based models, sizes range from 95K to 156M parameters.

\begin{figure}
  \centering
  \begin{subfigure}[b]{0.5\textwidth}
    \includegraphics[width=\textwidth]{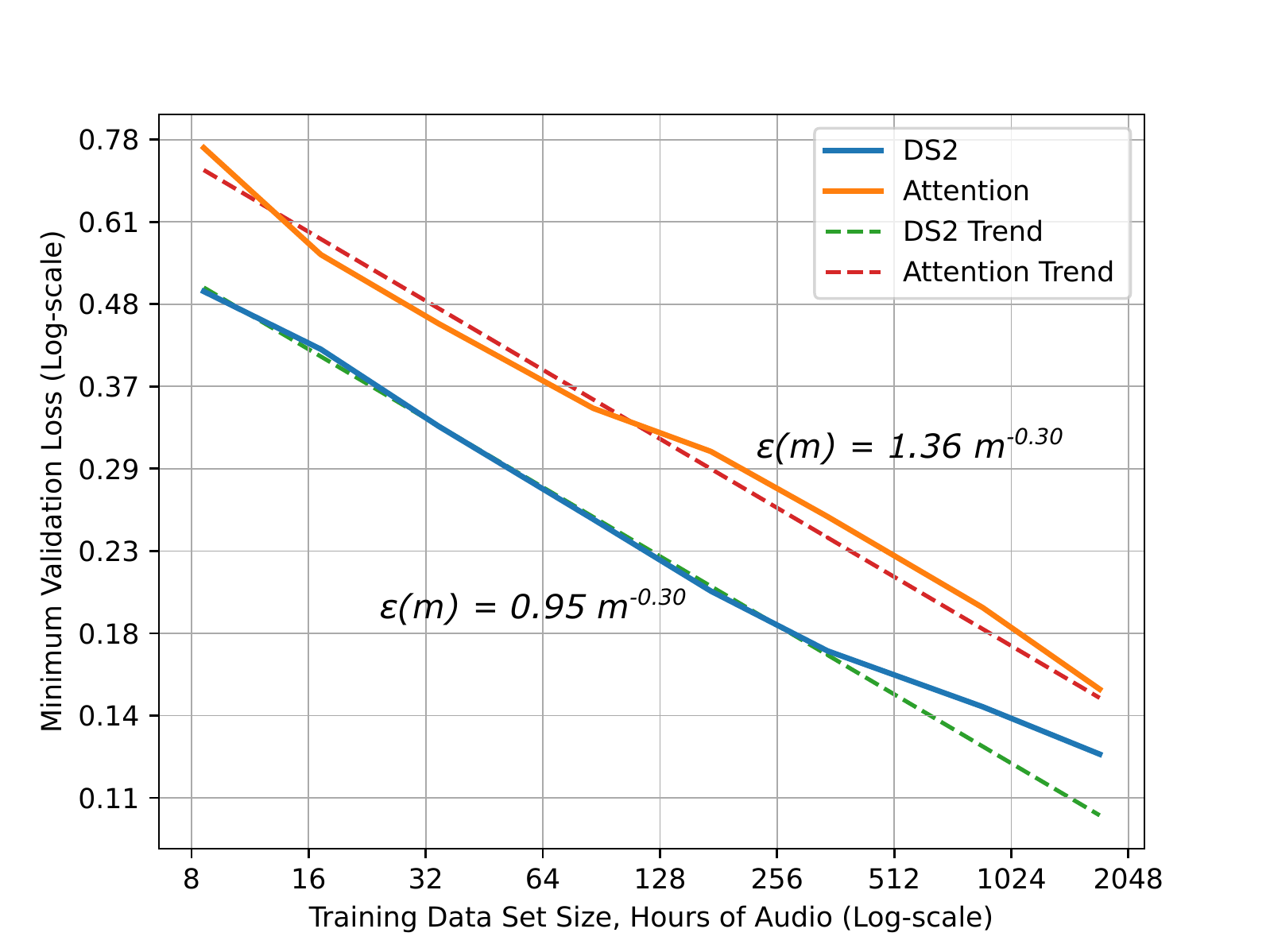}
    \label{fig:speech_data_gen}
  \end{subfigure}
  \hspace{-10pt}
  \begin{subfigure}[b]{0.5\textwidth}
    \includegraphics[width=\textwidth]{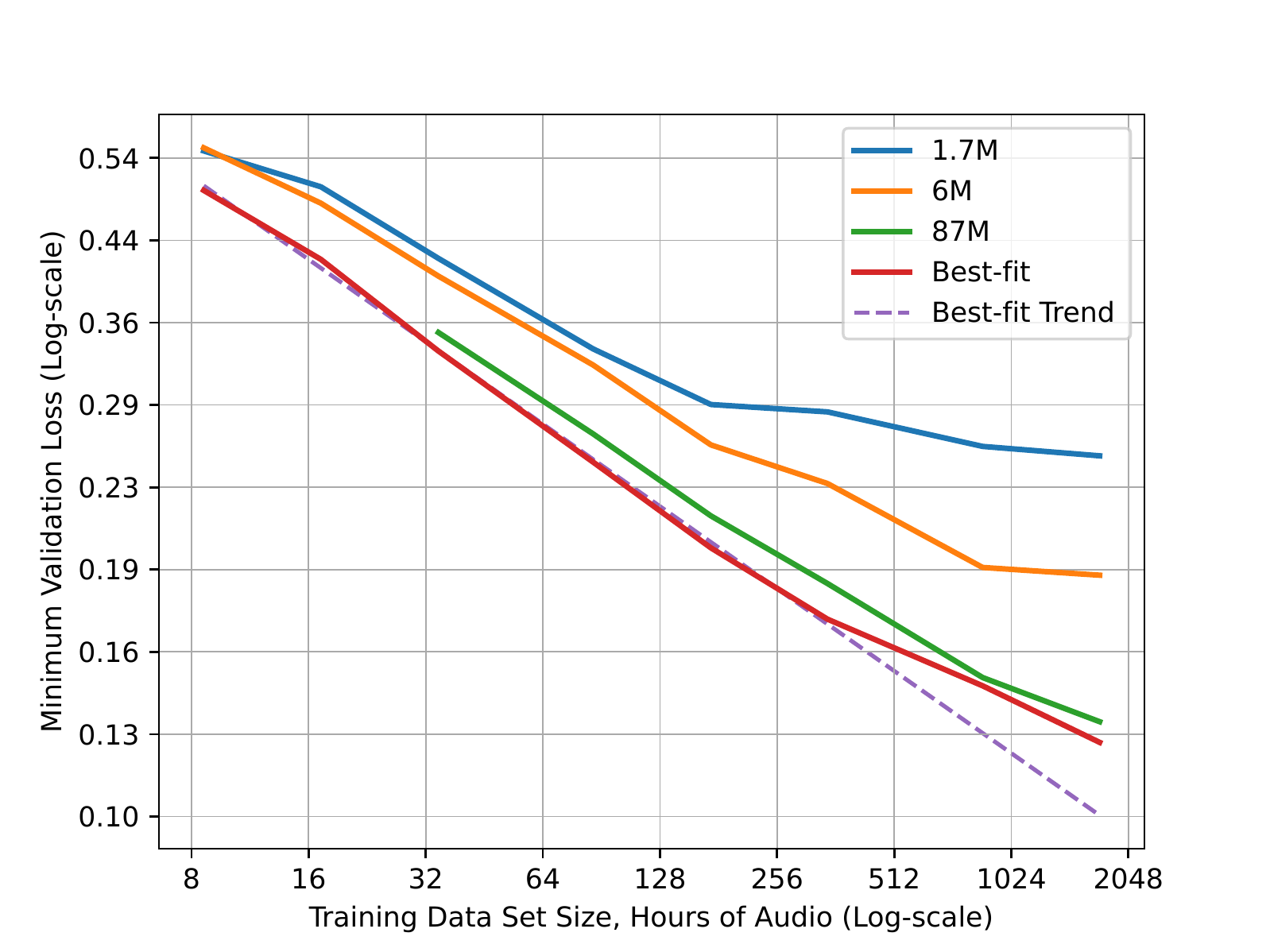}
    \label{fig:speech_model_size}
  \end{subfigure}
  \caption{Learning curves for DS2 and attention speech models (left), and learning curves for various DS2 model sizes, 1.7M to 87M parameters (right).}
  \label{fig:speech_results}
\end{figure}

Figure~\ref{fig:speech_results} (left) shows that both DS2 and attention based speech models experience the same power-law learning curve improvements. Although these models have significantly different encoders and decoders, they see the same relative improvements in character error rate as training set size increases with $\beta_g = -0.299 \pm 0.7\%$. Consistent with prior work (\cite{bahdanau:e2espeechattention:arxiv:2016,battenberg:speechtxducers:arxiv:2017}), larger attention models trained on larger data sets tend to be easier to optimize than DS2 models, whose generalization error tends away from the power-law trend on larger data sets.

For speech recognition, we trained a coarser spectra of model sizes, so model size scaling results for each training data size are not as meaningful as with LMs or image classification. Instead, we break down learning curves a bit by showing the curves for three different DS2 model sizes, 1.7M to 87M parameters (right side of Figure~\ref{fig:speech_results}). These curves show similar trends to those in other domains: As data size increases, most models experience power-law generalization improvements until the data size approaches their effective capacity. In this case, the 1.7M parameter model's accuracy plateaus starting at around 170 hours of audio, and the 6M parameter model plateaus around 860 hours of audio (i.e., roughly 5$\times$ more, which is similar to the difference in model size). Larger models (e.g., 87M parameters) show generalization error close to the best-fit trend up to larger data set sizes.

\section{Implications of Generalization Error and Model Size Scaling}
\label{sec:discussion}

Predictable learning curves and model size scaling indicate some significant implications on how DL could proceed. For machine learning practitioners and researchers, predictable scaling can aid model and optimization debugging and iteration time, and offer a way to estimate the most impactful next steps to improve model accuracy. Operationally, predictable curves can guide decision-making about whether or how to grow data sets or computation. Finally, these curves can be used to estimate compute requirements and guide system design and expansion. They underscore the importance of continued computational scaling.

%
\subsection{The Learning Curves of Real Applications}

\begin{figure}
  \centering
  \begin{subfigure}[b]{0.6\textwidth}
    \includegraphics[width=\textwidth]{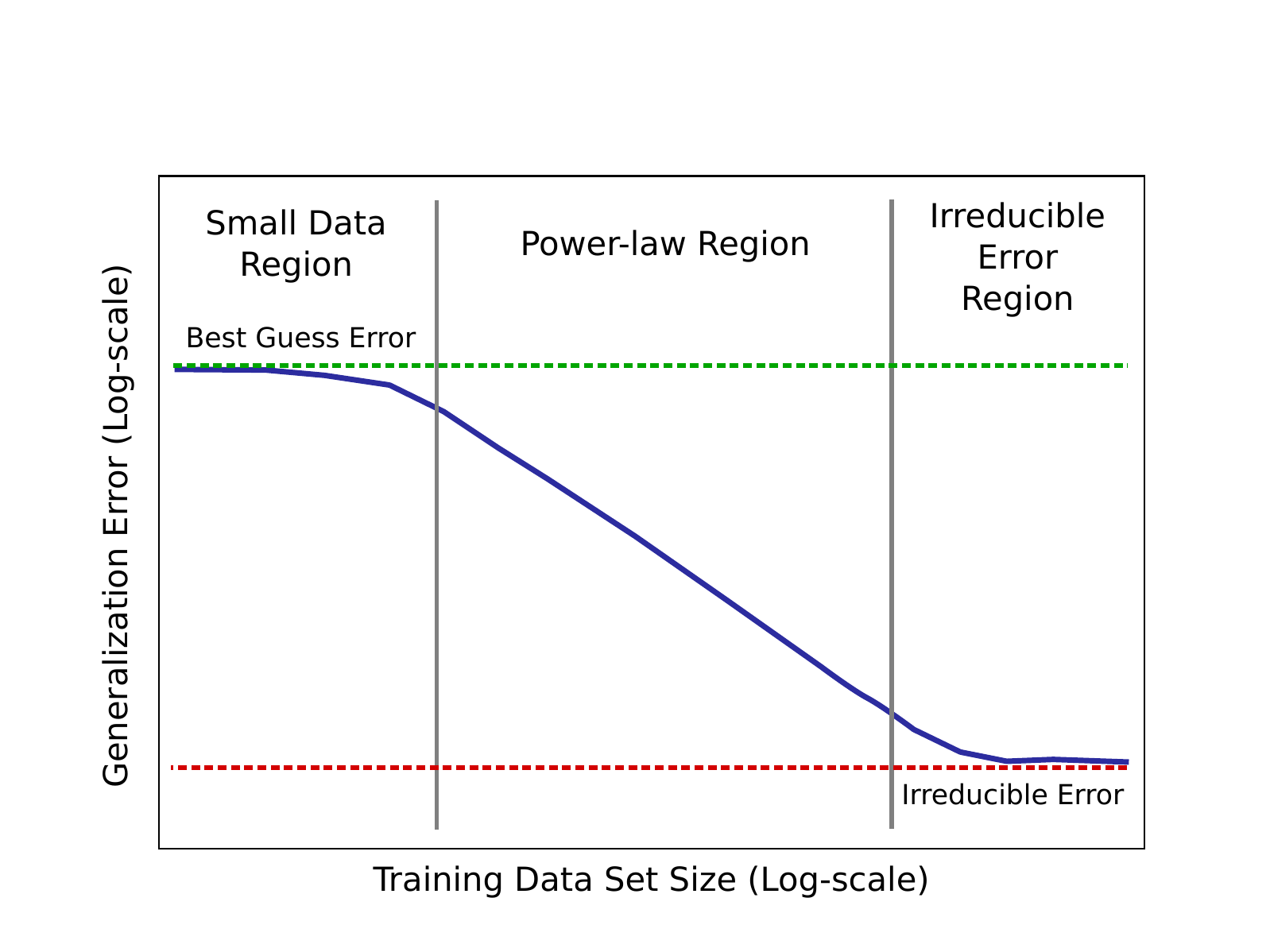}
    \label{fig:cartoon_power_law_regions}
  \end{subfigure}
  \vspace{-12pt}
  \caption{Sketch of power-law learning curves}
  \label{fig:cartoon_power_law}
  \vspace{-12pt}
\end{figure}

We start with a summary of the character of real application learning curves. Figure~\ref{fig:cartoon_power_law} shows a cartoon sketch power-law plot that breaks down learning curve phases for real applications. The curve begins in the \textbf{small data region}, where models will struggle to learn from a small number of training samples. Here, models can only perform as well as "best" or "random" guessing.

The middle portion of learning curves is the \textbf{power-law region}, where each new training sample provides information that helps models improve predictions on previously unseen samples. The power-law exponent defines the steepness of this curve, or the slope when viewed on a log-log scale. It is an indicator of the difficulty for models to represent the data generating function. Results in this paper indicate that the power-law exponent is unlikely to be easily predicted with prior theory and probably dependent on aspects of the problem domain or data distribution.

Finally, for most real world applications, there is likely to be a non-zero lower-bound error past which models will be unable to improve. This lower bound includes Bayes error---the information theoretic lower bound based on the data generating function---and a combination of other factors that cause imperfect generalization. For instance, mislabeled samples in the training or validation data sets are likely to cause irreducible error. We call this the \textbf{irreducible error region}. Although we have yet to reach the irreducible error region for real applications in this study, we have tested that this lower bound exists for toy problems.

%
\subsection{Implications for DL Practitioners and Researchers}

Our results indicate that in many real world contexts, simply scaling your training data set and models is likely to predictably improve the model's accuracy. This predictable behavior may help practitioners and researchers approach debugging and target better accuracy scaling.

\textbf{Debugging DL Training:}
The empirical learning curves that we collect show robust power-law regions. Surprisingly, we see a power-law region across all of our tests, which cover different problem domains, model architecture features, optimizers, and optimization functions. Table~\ref{table:breadth_of_factors} in Appendix~\ref{appendix:domaindetail} shows the breadth of architectural and optimization features in our tests.

Given the robustness of the power-law learning curve character, we suggest that DL practitioners and researchers consider this methodology for debugging data, model architecture, or optimization issues. Divergence from power-law-like improvements is likely to indicate deeper challenges with improving accuracy. For instance, when word and character language models began to diverge from power-law scaling for the 10\% and 2\% of the Billion Word benchmark, respectively, we saw this divergence as a cue to more exhaustively test hyperparameters. We found that larger training sets and larger models become harder to optimize. For large models with fixed hyperparameters, increasing the batch sizes and learning rates usually closed a significant portion of the gap to the power-law trend. Analogously, smaller training sets often require smaller batch sizes to ensure models behave well while fitting. We expect that other model debugging, such as finding good model priors or initialization, can also benefit from this methodology.


\textbf{Beating the Power-law:}
Machine learning researchers often try to improve model accuracy by changing model architectures trained on a given data set. Their efforts can involve complex trial-and-error and rely on creativity or epiphany to improve results. Our tests suggest that model architecture improvements such as model depth only shift learning curves down, but might not improve the power-law exponent.

A broader question is whether machine learning techniques could improve the power-law learning curve exponent, or in other words, to improve generalization more quickly as training data grows. Theory suggests that best case accuracy scaling is with $\beta_p = -0.5$ or $-1$. Thus, for some problem domains---especially language modeling---the potential accuracy improvements are immense if we knew ways to improve the power-law exponent.

We have yet to find factors that affect the power-law exponent. To beat the power-law as we increase data set size, models would need to learn more concepts with successively less data. In other words, models must successively extract more marginal information from each additional training sample. This might be difficult without adjustments to the data set. We suggest that future work more deeply analyze learning curves when using data handling techniques, such as data filtering/augmentation, few-shot learning, experience replay, and generative adversarial networks.

%
\subsection{Operational Implications}

Learning and model size curves can also guide decisions about data collection and scaling computation. As we project forward on learning curves, we can run into three types of scaling limits: training data is too small, computation is too slow, or irreducible error. 

\textbf{Model Exploration using Small Data:}
It may seem counterintuitive, but an implication of predictable scaling is that model architecture exploration should be feasible with small training data sets. Consider starting with a training set that is known to be large enough that current models show accuracy in the power-law region of the learning curve. Since we expect model accuracy to improve proportionally for different models, growing the training set and models is likely to result in the same relative gains across the models.

The possibility of doing small data testing has significant implications on manual and automatic architecture search. Researchers or DL systems may be able to iterate on small data sets to find models that can accurately model the structure of the data distribution. Then, these models can be scaled to larger data sets to ensure proportional accuracy gains.

Although small data set testing may be possible, it can be difficult to ensure that training data is large enough to see the power-law learning curve region. We have found that models with poor optimizer parameterization or model priors/initialization show accuracy cliffs, where accuracy is only as good as best guessing, but the model trains on enough data to be in the power-law region. Researchers must take great care when defining a "large enough" training set for small data testing. We leave the methodology for defining such a training set to future work.

\textbf{Computational Limits:}
If we have identified a desirable model to scale to larger training sets, the next potential limitation is the speed of computation. In some cases, training large models on very large data sets would take months or years of critical path compute time, making these training runs impractical for any real world problem on existing systems. However, predictable learning and model size curves may offer a way to project the compute requirements to reach a particular accuracy level. The compute requirements could inform decisions about how to scale computational capacity to unlock these compute-limited applications.

After reviewing the tests performed for this work, we find that we have run into compute limitations for the largest data sets of each application domain. Most frequently, we run out of GPU memory when trying to train the largest models on the largest data sets. In many cases, we could alleviate these issues with techniques like data or model parallelism, though they would require significant software changes to reduce per-compute-unit memory requirements. Alternatively, we could migrate training to systems with more memory. Further, our longest running training sessions have taken as long as 6 weeks to converge. Parallelism and hardware improvements to reduce this time are highly desirable.

\textbf{Running into Irreducible Error:}
If we approach the irreducible error region in real applications, improving accuracy will require techniques outside the straightforward recipe. As an example, reaching Bayes error for a problem would be an indicator that no further information can be extracted from the existing data set---the application might be considered "solved". If further model architecture search, training set growth, or computational scale cannot improve accuracy, it is likely that models are achieving the irreducible error. To improve error beyond this irreducible level may require techniques that could increase the information content of the data to distinguish between the samples that contribute to the Bayes error.

It may be difficult to assess whether we have reached irreducible error or if models just have inherent bias that makes them unable to resolve more information from the data. One approach might be to estimate the human error rate for the task. As long as humans are constrained to the same data for the problem, their best-case accuracy may be a reasonable upper bound on the irreducible error. If humans can perform better than current models, it is likely that models could be improved.

%
\subsection{Hardware Design Implications}

Since predictable learning and model size curves can offer a way to project the compute requirements required to reach a particular accuracy level, they can also help hardware developers predict the needs of DL hardware users.

\textbf{Deep Learning Hardware Design:}
First, there is a close tie from compute operation rate (e.g., floating point operations, or "FLOPs") to model accuracy improvements. Power-law learning curves and model size growth indicate that each new hardware generation with improved FLOP rate can provide a predictable step function improvement in relative DL model accuracy. Further, the different learning curve and model size growth exponents can act as an indicator of the computational scalability of different application domains. Different application domains will see varying benefits from improved FLOP rates, which can help prioritize the domains that should be targets for improved compute throughput.

Second, as new model architecture features emerge for DL applications, hardware designers can estimate the importance of accelerating these new model features. Suppose the new model feature runs very slowly on current hardware, and as a result, throughput is not sufficient for the new model architecture to improve SOTA (e.g., a new non-linearity not supported by current floating point function units). Implementing the new feature in hardware might be costly, and the resulting performance improvements might not provide the required throughput to achieve necessary model accuracy to improve the state-of-the-art. Hardware designers could estimate the throughput of a hardware implementation and the resulting model accuracy gains to weigh them against the benefits of other hardware components.

\textbf{The Performance-Accuracy Trade-off:}
Many DL software and hardware techniques impose a trade-off between model accuracy and the speed of computation. Learning curves and model size growth can indicate whether these techniques could regain lost accuracy by improving the speed of computation. For example, low-precision computation/quantization and sparse models give up some model accuracy (e.g., up to $20\%$) in order to improve compute throughput. If the compute throughput improvements allow DL developers to train larger models on larger data sets, these accuracy losses might be easily recoverable.

\section{Conclusion}

The deep learning (DL) community has created impactful advances across diverse application domains by following a straightforward recipe: search for improved model architectures, create large training data sets, and scale computation. While model architecture search can be unpredictable, the model accuracy improvements from growing data set size and scaling computation are empirically predictable. We empirically validate that DL model accuracy improves as a power-law as we grow training sets for state-of-the-art (SOTA) model architectures in four machine learning domains: machine translation, language modeling, image processing, and speech recognition. These power-law learning curves exists across all tested domains, model architectures, optimizers, and loss functions. Further, within each domain, model architecture and optimizer changes only shift the learning curves but do \textit{not} affect the power-law exponent---the "steepness" of the learning curve. We also show that model size scales sublinearly with data size. These scaling relationships have significant research, practice, and systems implications on deep learning progress.

\newpage
\renewcommand{\bibsection}{\subsubsection*{References}} \small
\bibliographystyle{abbrvnat}
\bibliography{bibliography}

\begin{thebibliography}{39}
\providecommand{\natexlab}[1]{#1}
\providecommand{\url}[1]{\texttt{#1}}
\expandafter\ifx\csname urlstyle\endcsname\relax
  \providecommand{\doi}[1]{doi: #1}\else
  \providecommand{\doi}{doi: \begingroup \urlstyle{rm}\Url}\fi

\bibitem[Amari(1993)]{amari:universallearningcurve:jnn:1993}
S.~Amari.
\newblock {A Universal Theorem on Learning Curves}.
\newblock \emph{Neural Networks}, 6:\penalty0 161--166, 1993.

\bibitem[Amari and Murata(1993)]{amari:entropicloss:neuralcomp:1993}
S.~Amari and N.~Murata.
\newblock {Statistical Theory of Learning Curves under Entropic Loss
  Criterion}.
\newblock \emph{Neural Computation}, 5\penalty0 (1):\penalty0 140--153, 1993.

\bibitem[Amari et~al.(1992)Amari, Fujita, and
  Shinomoto]{amari:fourlearningcurves:neuralcomp:1992}
S.~Amari, N.~Fujita, and S.~Shinomoto.
\newblock {Four Types of Learning Curves}.
\newblock \emph{Neural Computation}, 4\penalty0 (4):\penalty0 605--618, 1992.

\bibitem[Amodei et~al.(2016)Amodei, Anubhai, Battenberg, Case, Casper,
  Catanzaro, Chen, Chrzanowski, Coates, Diamos, et~al.]{amodei:ds2:icml:2016}
D.~Amodei, R.~Anubhai, E.~Battenberg, C.~Case, J.~Casper, B.~Catanzaro,
  J.~Chen, M.~Chrzanowski, A.~Coates, G.~Diamos, et~al.
\newblock {Deep Speech 2: End-to-End Speech Recognition in English and
  Mandarin}.
\newblock In \emph{Proceedings of The International Conference on Machine
  Learning (ICML)}, pages 173--182, 2016.

\bibitem[Arpit et~al.(2017)Arpit, Jastrz{\k e}bski, Ballas, Krueger, Bengio,
  Kanwal, Maharaj, Fischer, Courville, Bengio, and
  {Lacoste-Julien}]{arpit:memorizationdeep:icml:2017}
D.~Arpit, S.~Jastrz{\k e}bski, N.~Ballas, D.~Krueger, E.~Bengio, M.~S. Kanwal,
  T.~Maharaj, A.~Fischer, A.~Courville, Y.~Bengio, and S.~{Lacoste-Julien}.
\newblock {A Closer Look at Memorization in Deep Networks}.
\newblock In \emph{Proceedings of the International Conference on Machine
  Learning}, August 2017.

\bibitem[Bahdanau et~al.(2016)Bahdanau, Chorowski, Serdyuk, Brakel, and
  Bengio]{bahdanau:e2espeechattention:arxiv:2016}
D.~Bahdanau, J.~Chorowski, D.~Serdyuk, P.~Brakel, and Y.~Bengio.
\newblock {End-to-end Attention-based Large Vocabulary Speech Recognition}.
\newblock \emph{arXiv preprint arXiv:1508.04395v2}, 2016.

\bibitem[Banko and Brill(2001)]{banko:verylargenld:acl:2001}
M.~Banko and E.~Brill.
\newblock {Scaling to Very Very Large Corpora for Natural Language
  Disambiguation}.
\newblock In \emph{Proceedings of Association of Computational Linguistics
  (ACL)}, January 2001.

\bibitem[Bartlett and Mendelson(2002)]{bartlett:rademachergaussian:jmlr:2002}
P.~L. Bartlett and S.~Mendelson.
\newblock {Rademacher and Gaussian Complexities: Risk Bounds and Structural
  Results}.
\newblock In \emph{Journal of Machine Learning Research 3}, pages 463--482,
  November 2002.

\bibitem[Battenberg et~al.(2017)Battenberg, Chen, Child, Coates, Gaur, Li, Liu,
  Satheesh, Seetapun, Sriram, and Zhu]{battenberg:speechtxducers:arxiv:2017}
E.~Battenberg, J.~Chen, R.~Child, A.~Coates, Y.~Gaur, Y.~Li, H.~Liu,
  S.~Satheesh, D.~Seetapun, A.~Sriram, and Z.~Zhu.
\newblock {Exploring Neural Transducers for End-to-end Speech Recognition}.
\newblock \emph{arXiv preprint arXiv:1707.07413}, 2017.

\bibitem[Blumer et~al.(1989)Blumer, Ehrenfeucht, Haussler, and
  Warmuth]{blumer:learnvcdim:jacm:1989}
A.~Blumer, A.~Ehrenfeucht, D.~Haussler, and M.~K. Warmuth.
\newblock {Learnability and the Vapnik-Chervonenkis Dimension}.
\newblock \emph{Journal of the ACM (JACM)}, 36\penalty0 (4):\penalty0 929--965,
  October 1989.

\bibitem[Chelba et~al.(2013)Chelba, Mikolov, Schuster, Ge, Brants, Koehn, and
  Robinson]{chelba:1bw:arxiv:2013}
C.~Chelba, T.~Mikolov, M.~Schuster, Q.~Ge, T.~Brants, P.~Koehn, and
  T.~Robinson.
\newblock {One Billion Word Benchmark for Measuring Progress in Statistical
  Language Modeling}.
\newblock \emph{arXiv preprint arXiv:1312.3005}, 2013.

\bibitem[Chorowski et~al.(2015)Chorowski, Bahdanau, Serdyuk, Cho, and
  Bengio]{chorowski:speechattention:nips:2015}
J.~K. Chorowski, D.~Bahdanau, D.~Serdyuk, K.~Cho, and Y.~Bengio.
\newblock {Attention-based Models for Speech Recognition}.
\newblock In \emph{Advances in Neural Information Processing Systems (NIPS)},
  pages 577--585, 2015.

\bibitem[Collins et~al.(2017)Collins, {Sohl-Dickstein}, and
  Sussillo]{collins:modelcaptrain:iclr:2017}
J.~Collins, J.~{Sohl-Dickstein}, and D.~Sussillo.
\newblock {Capacity and Trainability in Recurrent Neural Networks}.
\newblock In \emph{Proceedings of The International Conference on Learning
  Representations (ICLR)}, 2017.

\bibitem[Dziugaite and Roy(2017)]{dziugaite:nonvacuousbounds:icml:2017}
G.~K. Dziugaite and D.~M. Roy.
\newblock {Computing Nonvacuous Generalization Bounds for Deep (Stochastic)
  Neural Networks with Many More Parameters than Training Data}.
\newblock In \emph{Proceedings of The International Conference on Machine
  Learning (ICML)}, 2017.

\bibitem[Ehrenfeucht et~al.(1989)Ehrenfeucht, Haussler, Kearns, and
  Valiant]{ehrenfeucht:sampcomplxlowerbd:infocomp:1989}
A.~Ehrenfeucht, D.~Haussler, M.~Kearns, and L.~Valiant.
\newblock {A General Lower Bound on the Number of Examples Needed for
  Learning}.
\newblock \emph{Information and Computation}, 82:\penalty0 247--261, 1989.
\newblock Workshop on Computing Learning Theory, 1988.

\bibitem[Graves et~al.(2006)Graves, Fern{\'a}ndez, Gomez, and
  Schmidhuber]{graves:ctc:icml:2006}
A.~Graves, S.~Fern{\'a}ndez, F.~Gomez, and J.~Schmidhuber.
\newblock {Connectionist Temporal Classification: Labelling Unsegmented
  Sequence Data with Recurrent Neural Networks}.
\newblock In \emph{Proceedings of the International Conference on Machine
  Learning (ICML)}, 2006.

\bibitem[Gy{\"o}rgyi and Tishby(1990)]{gyorgyi:learningarule:nnandspin:1990}
G.~Gy{\"o}rgyi and N.~Tishby.
\newblock {Statistical Theory of Learning a Rule}.
\newblock \emph{Neural Networks and Spin Glasses}, pages 3--36, 1990.

\bibitem[Hannun et~al.(2014)Hannun, Case, Casper, Catanzaro, Diamos, Elsen,
  Prenger, Satheesh, Sengupta, Coates, et~al.]{hannun:deepspeech:arxiv:2014}
A.~Hannun, C.~Case, J.~Casper, B.~Catanzaro, G.~Diamos, E.~Elsen, R.~Prenger,
  S.~Satheesh, S.~Sengupta, A.~Coates, et~al.
\newblock {Deep Speech: Scaling Up End-to-End Speech Recognition}.
\newblock \emph{arXiv preprint arXiv:1412.5567}, 2014.

\bibitem[Harvey et~al.(2017)Harvey, Liaw, and
  Mehrabian]{harvey:nearlytightvcdim:jmlr:2017}
N.~Harvey, C.~Liaw, and A.~Mehrabian.
\newblock {Nearly-tight VC-dimension Bounds for Piecewise Linear Neural
  Networks}.
\newblock In \emph{Proceedings of Machine Learning Research}, volume~65, pages
  1--–12, 2017.

\bibitem[Haussler(1988)]{haussler:valiantsframework:ai:1988}
D.~Haussler.
\newblock {Quantifying Inductive Bias: AI Learning Algorithms and Valiant's
  Learning Framework}.
\newblock \emph{Artificial Intelligence}, 36\penalty0 (2):\penalty0 177--221,
  1988.

\bibitem[Haussler et~al.(1996)Haussler, Kearns, Seung, and
  Tishby]{haussler:rigorousbounds:machinelearning:1996}
D.~Haussler, M.~Kearns, H.~S. Seung, and N.~Tishby.
\newblock {Rigorous Learning Curve Bounds from Statistical Mechanics}.
\newblock \emph{Machine Learning}, 25\penalty0 (2):\penalty0 195--236, November
  1996.

\bibitem[He et~al.(2016)He, Zhang, Ren, and Sun]{he:resnets:cvpr:2016}
K.~He, X.~Zhang, S.~Ren, and J.~Sun.
\newblock {Deep Residual Learning for Image Recognition}.
\newblock In \emph{Proceedings of the IEEE Conference on Computer Vision and
  Pattern Recognition (CVPR)}, pages 770--778, June 2016.

\bibitem[Jozefowicz et~al.(2016)Jozefowicz, Vinyals, Schuster, Shazeer, and
  Wu]{jozefowicz:lmlimits:arxiv:2016}
R.~Jozefowicz, O.~Vinyals, M.~Schuster, N.~Shazeer, and Y.~Wu.
\newblock {Exploring the Limits of Language Modeling}.
\newblock \emph{arXiv preprint arXiv:1602.02410v2}, 2016.

\bibitem[Kawaguchi et~al.(2017)Kawaguchi, Kaelbling, and
  Bengio]{kawaguchi:dlgeneralization:arxiv:2017}
K.~Kawaguchi, L.~P. Kaelbling, and Y.~Bengio.
\newblock {Generalization in Deep Learning}.
\newblock \emph{arXiv preprint arXiv:1710.05468v1}, October 2017.

\bibitem[Klein et~al.(2017)Klein, Kim, Deng, Senellart, and
  Rush]{klein:opennmt:acl:2017}
G.~Klein, Y.~Kim, Y.~Deng, J.~Senellart, and A.~M. Rush.
\newblock Opennmt: Open-source toolkit for neural machine translation.
\newblock In \emph{Proceedings of the Association for Computational Linguistics
  (ACL)}, 2017.

\bibitem[Koehn et~al.(2007)Koehn, Hoang, Birch, Callison-Burch, Federico,
  Bertoldi, Cowan, Shen, Moran, Zens, Dyer, Bojar, Constantin, and
  Herbst]{koehn:moses:jacl:2007}
P.~Koehn, H.~Hoang, A.~Birch, C.~Callison-Burch, M.~Federico, N.~Bertoldi,
  B.~Cowan, W.~Shen, C.~Moran, R.~Zens, C.~Dyer, O.~Bojar, A.~Constantin, and
  E.~Herbst.
\newblock {Moses: Open Source Toolkit for Statistical Machine Translation}.
\newblock In \emph{Proceedings of the Association of Computational Linguistics,
  Interactive Poster and Demonstration Sessions}, pages 177--180, 2007.

\bibitem[Luong et~al.(2017)Luong, Brevdo, and
  Zhao]{luong:seq2seqtutorial:tf:2017}
M.~Luong, E.~Brevdo, and R.~Zhao.
\newblock {Neural Machine Translation (seq2seq) Tutorial}.
\newblock \emph{https://github.com/tensorflow/nmt}, 2017.

\bibitem[Luong et~al.(2015)Luong, Pham, and
  Manning]{luong:globalattention:emnlp:2015}
T.~Luong, H.~Pham, and C.~D. Manning.
\newblock Effective approaches to attention-based neural machine translation.
\newblock In \emph{Proceedings of the Conference on Empirical Methods in
  Natural Language Processing (EMNLP)}, pages 1412--1421, 2015.

\bibitem[Russakovsky et~al.(2015)Russakovsky, Deng, Su, Krause, Satheesh, Ma,
  Huang, Karpathy, Khosla, Bernstein, Berg, and
  Fei-Fei]{russakovsky:imagenet:arxiv:2015}
O.~Russakovsky, J.~Deng, H.~Su, J.~Krause, S.~Satheesh, S.~Ma, Z.~Huang,
  A.~Karpathy, A.~Khosla, M.~Bernstein, A.~C. Berg, and L.~Fei-Fei.
\newblock {ImageNet Large Scale Visual Recognition Challenge}.
\newblock \emph{arXiv preprint arXiv:1409.0575}, January 2015.

\bibitem[Schwarze and Hertz(1993)]{schwarze:mechlargecommittee:nips:1993}
H.~Schwarze and J.~Hertz.
\newblock {Statistical Mechanics of Learning in a Large Committee Machine}.
\newblock In \emph{Advances in Neural Information Processing Systems, (NIPS)},
  pages 523--530, San Francisco, CA, USA, 1993. Morgan Kaufmann Publishers Inc.

\bibitem[Sennrich et~al.(2016{\natexlab{a}})Sennrich, Haddow, and
  Birch]{sennrich:bpe:arxiv:2016}
R.~Sennrich, B.~Haddow, and A.~Birch.
\newblock {Neural Machine Translation of Rare Words with Subword Units}.
\newblock \emph{arXiv preprint arXiv:1508.07909}, 2016{\natexlab{a}}.

\bibitem[Sennrich et~al.(2016{\natexlab{b}})Sennrich, Haddow, and
  Birch]{sennrich:edinburghnmtwmt16:arxiv:2016}
R.~Sennrich, B.~Haddow, and A.~Birch.
\newblock {Edinburgh Neural Machine Translation Systems for WMT 16}.
\newblock \emph{arXiv preprint arXiv:1606.02891}, 2016{\natexlab{b}}.

\bibitem[Seung et~al.(1992)Seung, Sompolinsky, and
  Tishby]{seung:mechlearningexamples:physreva:1992}
H.~S. Seung, H.~Sompolinsky, and N.~Tishby.
\newblock {Statistical Mechanics of Learning from Examples}.
\newblock \emph{Physical Review A}, 45:\penalty0 6056--6091, April 1992.

\bibitem[Smith and Le(2017)]{smith:bayesiangeneralize:arxiv:2017}
S.~L. Smith and Q.~V. Le.
\newblock {A Bayesian Perspective on Generalization and Stochastic Gradient
  Descent}.
\newblock \emph{arXiv preprint arXiv:1710.06451v2}, October 2017.

\bibitem[Sun et~al.(2017)Sun, Shrivastava, Singh, and
  Gupta]{sun:dataeffective:iccv:2017}
C.~Sun, A.~Shrivastava, S.~Singh, and A.~Gupta.
\newblock {Revisiting Unreasonable Effectiveness of Data in Deep Learning Era}.
\newblock In \emph{Proceedings of the International Conference on Computer
  Vision (ICCV)}, 2017.

\bibitem[Tucker(1959)]{tucker:glivenko-cantelli:mathstats:1959}
H.~G. Tucker.
\newblock {A Generalization of the Glivenko-Cantelli Theorem}.
\newblock \emph{The Annals of Mathematical Statistics}, 30\penalty0
  (3):\penalty0 828--830, September 1959.

\bibitem[Vapnik(1998)]{vapnik:vcdim:ieeenn:1998}
V.~Vapnik.
\newblock {An Overview of Statistical Learning Theory}.
\newblock In \emph{IEEE Transactions on Neural Networks}, volume~10, pages
  988--999, September 1998.

\bibitem[Zhang et~al.(2017)Zhang, Bengio, Hardt, Recht, and
  Vinyals]{zhang:rethinkgeneral:arxiv:2017}
C.~Zhang, S.~Bengio, M.~Hardt, B.~Recht, and O.~Vinyals.
\newblock {Understanding Deep Learning Requires Rethinking Generalization}.
\newblock \emph{arXiv preprint arXiv:1611.03530v2}, 2017.

\bibitem[Zilly et~al.(2017)Zilly, Srivastava, Koutník, and
  Schmidhuber]{zilly:rhns:icml:2017}
J.~G. Zilly, R.~K. Srivastava, J.~Koutník, and J.~Schmidhuber.
\newblock {Recurrent Highway Networks}.
\newblock In \emph{Proceedings of The International Conference on Machine
  Learning (ICML)}, 2017.

\end{thebibliography}

\newpage
\appendix \normalsize 
\newtheorem{theorem}{Theorem}

%
\section{Detail on Tested Machine Learning Domains\label{appendix:domaindetail}}

Based on the results presented in this paper, we suspect the power-law data-generalization behaviors of each machine learning domain are due to the structure of the problem domain. To encourage further theoretical investigation, this section reports precise definitions of input and output spaces, optimized and reported loss functions for each machine learning domain, and other information that may be relevant to predicting the data-generalization and model size scaling. Additionally, to show the breadth of our testing, Table~\ref{table:breadth_of_factors} summarizes the different domains, model architecture features, optimization and loss functions we tested.

\begin{table}[h]
    \centering
    \small
    \caption{Breadth of domains, model features, optimizers, loss functions tested}
    \begin{tabular}{|l|l|l|l|l|l|}
        \hline
         & & & & Loss & \\
        Domain & Model & Model Features & Optimizer & Function & Exponent \\
        \hline
        Machine & LSTM & Encoder-decoder with attention, & Adam & Token & $-0.128$ \\
        Translation & & with and without dropout & & Error & \\
        \hline
        Word LMs & LSTM & GEMMs, $\sigma$+$tanh$ non-linearities & SGD & Xentropy & $-0.066$ \\
                 \cline{2-6}
                 & RHN  & GEMMs, $\sigma$+$tanh$ non-linearities & SGD & Xentropy & $-0.070$ \\
        \hline
        Char LMs & RHN & GEMMs, $\sigma$+$tanh$ non-linearities & SGD, Adam & Xentropy & $-0.094$ \\
        \hline
        Image & ResNet & Feed-forward, CONV blocks, & Nesterov & Classify & $-0.309$ \\
        Classification & & pooling and skip connections & Momentum & Error & \\
        \cline{5-6}
        & & & & X-entropy & $-0.350$ \\
        \hline
        Speech & DS2 & Bi-LSTM, CTC loss & Adam & CER & $-0.299$ \\
        \cline{2-6}
        Recognition & Attention & Bi-LSTM, CONVs, attention layer & Adam & CER & $-0.296$ \\
        \hline
    \end{tabular}
    \label{table:breadth_of_factors}
\end{table}

\subsection{Neural Machine Translation}

Given input and output vocabularies, $V_S$ and $V_T$, NMT models learn a mapping $D_S \rightarrow D_T$ where $D_\cdot = V_\cdot^*$ (Kleene star). In this work, we use a word-piece vocabulary shared between the source and target languages. After applying pre-processing methods\footnote{clean-up and byte pair encoding uses \href{https://github.com/tensorflow/nmt/blob/master/nmt/scripts/wmt16_en_de.sh}{Tensorflow NMT WMT scripts}} adopted in many SOTA models, there are 36545 sub-word tokens. We include UNK and PAD tokens for unknown words and minibatch padding for the source domain (German, $|V_S|=36547$); for the target domain (English), UNK, PAD, SOS (start-of-sequence), and EOS (end-of-sequence) are included ($|V_T|=36549$). The German and English sentences in newstest2016 were on average 27 and 25 tokens long with the longest sequences having 101 and 94 tokens respectively.

During training, we minimize cross entropy loss (i.e. the conditional probability of the target sentence given the source sentence). We report per-token error rate and bits-per-token. Because our reported metrics are per-token measure of the target language, the dataset size is given by the number of English tokens in the training set.

\subsection{Language Modeling}

\subsubsection{Word Language Models}
During training for world language models, we unroll sequences out to length 80 for backpropagation. We also use continuous minibatching: At end of one sentence in the data set, we concatenate an end-of-sentence designator, followed by the next sentence from the data set.

Let $C$ be the language's vocabulary. Then, $|C| = 10,004$ after we include special symbols like the unknown token. The input space is $I = \bigcup C^i$ where $i$ is the number of words previously seen in a sequence. We use continuous minibatching, so the effective history length, $i$, can be very long. The output space is $O = C$.

Rather than perplexity, we use normalized cross-entropy loss: $-\frac{1}{N}\sum_{i} ln~p_{w_i}$, where $p_{w_i}$ is the model's predicted probability of seeing the $i$th token. $N$ is either the number of sequences in a batch for training optimization or $N$ is the number of predicted words in the validation set.

\subsubsection{Character Language Models}
For character language models, we unroll sequences out to length 150 characters. Unlike word language models, we use non-continuous minibatching, so some sequences end at an end-of-sentence token. Sequences longer than 150 characters are truncated.

Let $C$ be the language's vocabulary of alphanumeric characters and symbols. Then, $|C| = 98$ after we include special symbols like the end-of-sentence token. Similar to word language models, the input space is $I = \bigcup C^i$ where $i$ is the number of characters previously seen in a sequence. Since we use non-continuous minibatching, so the effective history length, $i$, is at most 150. The output space is $O = C$.

Similar to word language models, we use normalized cross-entropy loss: $-\frac{1}{N}\sum_{i} ln~p_{w_i}$, where $p_{w_i}$ is the model's predicted probability of seeing the $i$th token. $N$ is either the number of sequences in a batch for training optimization or $N$ is the number of predicted characters in the validation set.

\subsection{Image Classification}

ImageNet images were initially scaled proportionally so that the shortest dimension of the image is 256 pixels. During training, these images are cropped to 224x224 as input to the CNN. Input images are 224x224 pixels by 3 color channels of 8 bits each. Thus, the total input space size is $|I| = 224 * 224 * 3 * 256 \approx 38.5M$. The output space is 1,000 different object classes that might be contained in the image. For training, we also augment the dataset by modifying the brightness, contrast, saturation, and lighting. In addition, we also flip the image horizontally. \footnote{Training and data augmentation is performed using ResNet implementation in \href{https://github.com/ppwwyyxx/tensorpack/blob/master/examples/ResNet/imagenet-resnet.py}{TensorPack}}

We optimize for classification cross-entropy loss on each training image, and we report average validation cross-entropy, top-1, and top-5 classification error. Each loss calculation still follows the power-law. However, we note that top-k classification error ($k>1$) is not a distance metric; It uses set containment, which is not symmetric. Alternatively, it is a product of distance metrics, which is not necessarily a distance metric.

\subsection{Speech Recognition}

The audio input to speech recognition models can be represented as the sequence $x = (x_1, .., x_t)$ of length $t$. Each $x_i$ is an audio spectrogram over a small time window. Each predicted output is a character, encoded as a one-hot vector, $y_i$, representing the most probable text token at sequence step $i$. Output sequences are of the form $y = (y_1, ..., y_u)$. Models predict the conditional distribution $p(y|x)$ using an encoder-decoder form. Thus, $p(y|x) = \text{Decode}(\text{Encode}(x), y)$.

\subsubsection{Deep Speech 2}

In DS2 model, the encoder is represented by a stack of recurrent layers with LSTM cells and the decoder is the connectionist temporal classification (CTC)  (\cite{graves:ctc:icml:2006}). The CTC loss function computes the conditional probability by marginalizing all possible alignments and it assumes conditional independence between output predictions at different time steps given aligned inputs. An extra blank label, which can be interpreted as no label, is introduced to map $h$ and $y$ to the same length (i.e., an alignment or path). $a$ is obtained by inserting ($t'$ - $u$) blanks into $y$. A mapping $\mathcal{B}: a \to y$ is defined between $a$ and $y$, which can be done by removing all blanks and repeating letters in $a$.
  \begin{align}
    P_{\text{CTC}}(y|x) & = \sum_{a \in \mathcal{B}^{-1}(y)} P(a|h) \\
    & = \sum_{a \in \mathcal{B}^{-1}(y)}\prod_{t=1}^{t'}{P(a_t|h_t)} \\
    P(a_t | h_t) &= \text{softmax}(a_t, h_t)
  \end{align}

\subsubsection{Attention Model}

Similar to the DS2 model, the attention model uses a stack of recurrent layers with GRU cells as the encoder. The decoder consists of an attention layer followed by a recurrent layer. The attention mechanism aligns the input sequence to the output sequence. The attention mechanism removes the conditional independence assumption in output sequence that the DS2 model makes. More model, attention mechanism, and loss function details can be found in \cite{battenberg:speechtxducers:arxiv:2017}.


%
\section{Power-law learning curve for counting model classifier\label{appendix:powerlawproof}}

First, we show that the expected generalization error for a counting model decreases as a power-law with the size of number of training samples it observes. This proof inspects the asymptotic rate of convergence of the Glivenko-Cantelli theorem limit (\cite{tucker:glivenko-cantelli:mathstats:1959}). Some machinery:

Let $\mathcal{X} = \{0,1\}$ be the input space for a binary coin-flip probability estimator. Let $P_{true}: \mathcal{X} \rightarrow \mathbb{R}$ be the true model probability. To begin, we assume $P_{true}[0] = P_{true}[1] = 0.5$ (i.e., a fair coin flip), but our results easily generalize to unfairly weighted coins.

Let the training sets be $T_i$, such that $T_i$ contains $i$ iid samples from $P_{true}$. Further, let $T_i(x) = \{y \in T_i : y = x\}$ be the subset of samples in $T_i$ equal to $x$.

To start with, we will observe the learning behavior of a counting model, which approximates $P_{true}[x]$ by counting the proportion of training samples in $T_i$ that are equal to $x$. Thus, $P_i[x] = \frac{|T_i(x)|}{i}$.

Also to start with, let the model loss calculation be $l(P_i[x], P_{true}[x]) = |P_i[x] - P_{true}[i]|$ be the L1-norm. This proof sequence can be easily generalized to other loss functions including L2-norm and absolute KL-divergence, and we have empirically validated these norms show the same power-law behavior.

Finally, we define the total loss function as the weighted average loss per output prediction:
\begin{equation}
  L_i := \sum_{x\in\mathcal{X}} l(P_i[x], P_{true}[x]) * P_{true}[x]
\end{equation}

\begin{theorem}
  The expected total loss for a counting model trained on $T_i$ sampled from a true distribution fair coin flip is a power-law with exponent $-0.5$. Specifically,
  \begin{equation}
    \mathbb{E}[L_i] = \Omega\left(\frac{1}{\sqrt{2\pi i}}\right)
  \end{equation}
\end{theorem}

\begin{proof}
  First, we enumerate the $2^i$ possible ordered samples $T_i$, and we name them uniquely as $T_{i,j}$ for $j = 0,1,\ldots,2^i-1$. Let $P_{i,j}$ be the probability distribution predicted by a counting model trained with the set $T_{i,j}$.

  Now, we can expand the expectation as a sum:
  \begin{equation*}
    \begin{split}
      \mathbb{E}[L_i] & = \sum_{j=0}^{2^i-1} \left[ P[obtaining~T_{i,j}] * L_{i,j} \right] \\
      & = \sum_{j=0}^{2^i-1} \left[ P[obtaining~T_{i,j}] * \sum_{x\in\mathcal{X}} l(P_{i,j}[x], P_{true}[x]) * P_{true}[x] \right] \\
      & = \sum_{j=0}^{2^i-1} \left[ P[obtaining~T_{i,j}] * \sum_{x\in\mathcal{X}} |P_{i,j}[x] - P_{true}[x]| * P_{true}[x] \right] \\
    \end{split}
  \end{equation*}
  Exploiting the symmetry of the fair coin flip and plugging in values for $P_{true}$, we can simplify this to:
  \begin{equation*}
    \begin{split}
      \mathbb{E}[L_i] & = \frac{1}{2^i} \sum_{j=0}^{2^i-1} |P_{i,j}[x] - P_{true}[x]| \\
    \end{split}
  \end{equation*}
  We note that $T_{i,j} = T_{i,k}$ for $j \neq k$ iff the $j$th and $k$th samples each contain the same number of instances of $x\in\mathcal{X}$. In that case, $\forall x, P_{i,j}[x] = P_{i,k}[x]$. Further, note that there are ${i\choose{k}}$ sets, $T_{i,j}$, such that $|T_{i,j}[x]| = k$. We apply this counting argument to calculate the number $T_{i,j}$ that are equal. Let $k$ be the number of instances of $x=0$ in each set of training sets:
  \begin{equation*}
    \begin{split}
      \mathbb{E}[L_i] & = \frac{1}{2^i} \sum_{k=0}^{i} {i\choose{k}} \left\lvert\frac{i-k}{i} - \frac{1}{2}\right\rvert \\
      & = \frac{1}{2^i} \sum_{k=0}^{i} {i\choose{k}} \left(\frac{i-k}{i} - \frac{1}{2}\right) \\
      & = \frac{2}{2^i} \sum_{k=0}^{\lfloor\frac{i+1}{2}\rfloor} {i\choose{k}} \left(\frac{1}{2} - \frac{k}{i}\right) \\
      & = \frac{2}{2^i} \sum_{k=0}^{\lfloor\frac{i+1}{2}\rfloor} {i\choose{k}} \left(\frac{1}{2} - \frac{k}{i}\right) \\
      & = \frac{2}{2^i} \left[\frac{1}{2} \sum_{k=0}^{\lfloor\frac{i+1}{2}\rfloor} {i\choose{k}} - \sum_{k=0}^{\lfloor\frac{i+1}{2}\rfloor} \frac{k}{i} {i\choose{k}} \right] \\
      & = \frac{2}{2^i} \left[\frac{1}{2} \sum_{k=0}^{\lfloor\frac{i+1}{2}\rfloor} {i\choose{k}} - \sum_{k=0}^{\lfloor\frac{i-1}{2}\rfloor} {{i-1}\choose{k}} \right] \\
      & = \left\{\begin{array}{ll}
        \frac{1}{2^{i+1}} {i\choose{\frac{i}{2}}} & i~$even$ \\
        \frac{1}{2^i} {{i-1}\choose{\frac{i-1}{2}}} & i~$odd$ \\
      \end{array}\right.
    \end{split}
  \end{equation*}
  These last steps use the observation that summing half of a set of binomial coefficients gives roughly half of $2^i$:
  \begin{equation*}
    \sum_{k=0}^{\lfloor\frac{i+1}{2}\rfloor} {i\choose{k}} = \left\{\begin{array}{ll}
      2^{i-1} & i~$odd$ \\
      2^{i-1} - \frac{1}{2} {i\choose{\frac{i}{2}}} & i~$even$ \\
    \end{array}\right.
  \end{equation*}
  At this point, note that for $i$ even, we have that $\mathbb{E}[L_i] = \mathbb{E}[L_{i+1}]$. Thus, to bound $\mathbb{E}[L_i]$, it suffices to show that it is bounded for $i$ even.

  Finally, we use Sterling's factorial approximation, $i! = \Omega(\sqrt{2\pi i}\left(\frac{i}{e}\right)^i)$, to provide the desired bound:
  \begin{equation*}
    \begin{split}
      \mathbb{E}[L_i] = \frac{1}{2^{i+1}} {i\choose{\frac{i}{2}}} & = \Omega\left(\frac{\sqrt{2\pi i}(i/e)^i}{2^{i+1}\left(\sqrt{2\pi (i/2)}\left(\frac{i/2}{e}\right)^{i/2}\right)^2}\right) \\
      & = \Omega\left(\frac{(i/e)^i}{2^i\sqrt{2\pi i}\frac{1}{2^i}\left(\frac{i}{e}\right)^i}\right) \\
      & = \Omega\left(\frac{1}{\sqrt{2\pi}}i^{-0.5}\right)
    \end{split}
  \end{equation*}
\end{proof}

\end{document}